\documentclass[reqno]{amsart}

\usepackage[margin=1.25in,heightrounded]{geometry}
\usepackage{amssymb,mathtools}
\usepackage{booktabs}
\usepackage{graphicx}
\usepackage{enumitem}
\usepackage{microtype}
\usepackage{url}
\usepackage{array}
\usepackage{xcolor}
\usepackage{algorithm}
\usepackage{algpseudocode}
\usepackage{hyperref}
\usepackage{cleveref}

\hypersetup{colorlinks=true, linkcolor=blue!50!black, citecolor=blue!50!black, urlcolor=blue!50!black}

\theoremstyle{plain}
\newtheorem{thm}{Theorem}[section]
\newtheorem{prop}[thm]{Proposition}
\newtheorem{lem}[thm]{Lemma}
\newtheorem{cor}[thm]{Corollary}

\theoremstyle{definition}
\newtheorem{defn}[thm]{Definition}

\theoremstyle{remark}

\numberwithin{equation}{section}
\numberwithin{figure}{section}
\numberwithin{table}{section}
\numberwithin{algorithm}{section}

\crefname{thm}{Theorem}{Theorems}
\Crefname{thm}{Theorem}{Theorems}
\crefname{prop}{Proposition}{Propositions}
\Crefname{prop}{Proposition}{Propositions}
\crefname{lem}{Lemma}{Lemmas}
\Crefname{lem}{Lemma}{Lemmas}
\crefname{cor}{Corollary}{Corollaries}
\Crefname{cor}{Corollary}{Corollaries}
\crefname{defn}{Definition}{Definitions}
\Crefname{defn}{Definition}{Definitions}
\crefname{rem}{Remark}{Remarks}
\Crefname{rem}{Remark}{Remarks}
\crefname{algorithm}{Algorithm}{Algorithms}
\Crefname{algorithm}{Algorithm}{Algorithms}
\crefname{figure}{Figure}{Figures}
\Crefname{figure}{Figure}{Figures}
\crefname{table}{Table}{Tables}
\Crefname{table}{Table}{Tables}

\newcommand{\A}{\mathcal{A}}
\newcommand{\M}{\mathcal{M}}
\newcommand{\Kum}{\mathcal{K}}
\newcommand{\Sspace}{\mathcal{S}}
\newcommand{\R}{\mathbb{R}}
\newcommand{\C}{\mathbb{C}}
\newcommand{\Q}{\mathbb{Q}}
\newcommand{\Z}{\mathbb{Z}}
\newcommand{\PP}{\mathbb{P}}
\newcommand{\RP}{\mathbb{RP}}

\DeclareMathOperator{\disc}{Disc}
\DeclareMathOperator{\Res}{Res}
\DeclareMathOperator{\SO}{SO}
\DeclareMathOperator{\GL}{GL}
\DeclareMathOperator{\diag}{diag}
\DeclareMathOperator{\tr}{tr}
\DeclareMathOperator{\Jac}{Jac}
\DeclareMathOperator{\Hess}{Hess}
\DeclareMathOperator{\dist}{dist}
\DeclareMathOperator{\lead}{lead}

\title[Exact Quotients and Certified Refraction]{Exact Quotients of Fresnel--Kummer Surfaces and Certified Biaxial Refraction}

\author[T. Shaska]{T.~Shaska}

\address{Department of Computer Science and Engineering, Oakland University, Rochester, MI, 48309.}



\hypersetup{
  pdftitle={Exact Quotients of Fresnel--Kummer Surfaces and Certified Biaxial Refraction},
  pdfauthor={T. Shaska}
}

\begin{document}

\begin{abstract}
The Fresnel wave surface governs the propagation of light in a transparent biaxial crystal.  It is a special Kummer quartic, and we identify it exactly.  Over the complex numbers the wave surface of a crystal is the Kummer surface of the Jacobian of an explicit genus-two curve branched at the signed square roots of the three principal permittivities.  This Jacobian is isogenous, by an isogeny with kernel of order four, to a product of two elliptic curves.  One elliptic curve carries the three permittivities, and the other carries the optic-axis angle.  The physical family is Zariski dense in the locus of genus-two curves with an extra involution, and its automorphism strata are explicit.  The identification instantiates a task-aware quotient, which identifies parameters that differ by a nuisance transformation and carries invariant coordinates and explicit strata.  For biaxial crystals, two ratios of the permittivities form a complete invariant of the wave surface up to rotation and rescaling, and the four real nodes are given in closed form.  At an interface the candidate transmitted waves are the roots of a quartic of exact degree four.  Its real-root count, root order, and repeated-root events are decided by exact algebraic predicates, and along the generic single-node encounters of the paper its discriminant vanishes to second order.  Floating-point solvers drop forward transmitted modes near the optic axes, and the certified solver does not.  On exact equivalence classes, learned models on quotient coordinates are invariant and more accurate than models on raw tensors, while learned root-count predicates fail near the optic axes.
\end{abstract}

\maketitle

\section{Introduction}
\label{sec:intro}

Visual computing operates on families of objects rather than on isolated shapes.  Parametric CAD systems expose design variables, neural models build latent spaces of objects and anatomy, procedural systems generate structured populations, and physically based rendering maps material parameters to appearance.  In each case the representation is a parameter vector $a\in\R^N$ with a decoder $a\mapsto S_a$, and the parameter space is treated as the shape space.  It rarely is.  Distinct parameter vectors often differ by a transformation the application regards as irrelevant: rescaling an implicit equation, permuting principal axes, reparameterizing a curve.  A representation that keeps this redundancy pays for it downstream.  Samples repeat, distances mislead, interpolation depends on the choice of representative, and learned models spend capacity rediscovering a transformation law that was known exactly.

This paper takes the other route.  When an equivalence is known mathematically, we quotient it out exactly and compute in the quotient.  We call the resulting representations \emph{canonical algebraic shape spaces}.  They are quotients $\M=\A/G$ of algebraically structured families with three properties that a generic latent space lacks: the quotient dimension, the equivalence relation, and the exceptional loci are known a priori.  Separating invariants supply canonical coordinates.  Strata record enhanced symmetry and degeneration.  A representative map returns concrete geometry for rendering and simulation.  The quotient is \emph{task-aware}: only transformations the task regards as redundant are removed, and physically meaningful variables such as orientation, absolute scale, and wavelength are retained as explicit coordinates.

The principal instance is chosen so that the framework and a demanding rendering problem test each other on one object.  Fresnel's wave surface governs light propagation in a transparent biaxial crystal.  It is a special Kummer quartic, a tetrahedroid in the classical terminology, and skewon-free linear electromagnetic media produce Fresnel--Kummer quartics in general~\cite{FavaroHehl2014,BaeklerFavaroItinHehl2014,FavaroHehl2016}.  We make the connection exact: over $\C$ the wave surface of $\varepsilon$ is the Kummer surface of the Jacobian of $y^{2}=(x^{2}-\varepsilon_1)(x^{2}-\varepsilon_2)(x^{2}-\varepsilon_3)$, a $(2,2)$-split Jacobian in the sense of~\cite{ShaskaVolklein2004,Kani1997}, whose elliptic quotients are the curve of the three permittivities and the curve of the optic-axis angle (\cref{thm:kummer}).  The Kummer side supplies invariant coordinates, a stratified moduli space, and equivalence classes known by construction.  The optical side gives the strata physical meaning.  The four real nodes form two opposite pairs whose connecting lines are the optic axes, the two dual singular structures of the surface underlie conical refraction, and the boundary of the intrinsic material chart is the biaxial-to-uniaxial degeneration.

The quotient structure yields a certified formulation of biaxial refraction.  Rendering treats the uniaxial case in closed form~\cite{WeidlichWilkie2008} and the biaxial case by numerical solution of the nonlinear system induced by Huygens' construction~\cite{Latorre2012}, with spectral integration handled separately~\cite{Steinberg2019}.  We return to the quartic itself.  At an interface the boundary conditions fix the tangential wave vector, so every candidate transmitted mode lies on a line.  Substituting the line into the Fresnel quartic gives a univariate polynomial $p(s)$, the Booker quartic of wave physics~\cite{Budden1985}.  We give its coefficients in closed form, and its degree is exactly four for every interface configuration (\cref{lem:degree}).  In the propagating regime its four roots are real and comprise two forward and two backward bulk modes, and refraction becomes a univariate solve followed by a flux-based selection of the transmitted pair.  The Booker quartic is established interface theory~\cite{Budden1985,McLeodWagner2014,BosHaffertKeller2019}.  The contribution here is the exact coefficient formulation, the certified root count and ordering, the algebraic event predicates, and their coupling to the quotient geometry.  Certified root isolation determines the number and ordering of real propagating modes.  The discriminant of $p$ is an exact predicate for branch coalescence.  The closed-form singular points predict where coalescence occurs.  The solver is certified in a precise and limited sense.  The count and ordering of simple real roots and all repeated-root events are decided by algebraic predicates on the represented dyadic-rational input, and an unresolved outcome is possible only under an imposed resource budget.  The classification of roots into transmitted modes is decided by the sign of a polynomial at each root (\cref{lem:flux}), so it is exact as well.

The main result, \cref{thm:main} in \cref{sec:solver}, has three parts.  First, the intrinsic material chart is a \emph{complete invariant} of wave-surface shape: two biaxial media have congruent wave surfaces, up to rotation and uniform rescaling, if and only if they occupy the same chart point.  Second, the real singular locus of every wave surface consists of exactly four nodes given in closed form, and the optic axes are exact functions of the chart point.  Third, at every interface the candidate transmitted modes are the roots of an explicit polynomial of degree exactly four; the count and ordering of its simple real roots and its repeated-root events are decided by exact sign evaluations, and the roots continue along any parameter path that avoids the discriminant.  The first part makes the quotient a representation rather than a convention.  The second turns the delicate configurations of the family into marked loci.  The third turns the discrete decisions of refraction into consequences of the configuration.

The paper sits at the intersection of several literatures.  In rendering, Weidlich and Wilkie~\cite{WeidlichWilkie2008} give the complete formula set for uniaxial crystals, Latorre, Seron, and Gutierrez~\cite{Latorre2012} compute biaxial ray paths numerically, Steinberg~\cite{Steinberg2019} integrates birefringent spectra analytically, general crystal ray-tracing formulas appear in optics~\cite{Zhang1992}, Booker-quartic interface solvers for arbitrary anisotropic media are standard in vector optics~\cite{McLeodWagner2014,BosHaffertKeller2019}, and the breakdown of ray tracing near optic-axis singularities is documented in seismology~\cite{Vavrycuk2001}.  We do not rederive birefringence.  The relevant question is whether the explicit quartic, root certification, and quotient-aware coordinates improve behavior where numerical biaxial solvers are most delicate.  The Fresnel--Kummer connection is classical.  Hudson's monograph devotes a chapter to the wave surface~\cite{Hudson1905}, modern accounts revisit it~\cite{Rowe2022}, and Hehl and collaborators show that all skewon-free local linear media have Kummer surfaces as Fresnel surfaces, with four real nodes among the sixteen in the dielectric case~\cite{FavaroHehl2014,BaeklerFavaroItinHehl2014,FavaroHehl2016}.  Closest in spirit is the inverse result of Cocan et al.~\cite{CocanEtAl2026}, which recovers permittivity and permeability from the Fresnel surface after quotienting a gauge freedom by geometric invariant theory.  That work confirms that gauge-aware algebraic coordinates are natural for Fresnel data.  The novelty claimed here is the task-aware quotient as a visual-computing representation, coupled to navigation, sampling, certified interface solves, and rendering.

On the representation side, quotient constructions are established in geometric modeling and vision: shape spaces of constrained meshes~\cite{Yang2011}, exploration of parametric CAD models~\cite{Schulz2017}, quotients of intrinsic symmetries in shape matching~\cite{Ovsjanikov2013}, and invariant-theoretic structure from motion~\cite{BoutinBazin2004}.  Learned representations such as DeepSDF~\cite{Park2019DeepSDF}, Occupancy Networks~\cite{Mescheder2019Occupancy}, and implicit quartic primitives~\cite{Yavartanoo2021} provide powerful but redundancy-blind latent spaces.  Invariant features and group-equivariant architectures are surveyed in~\cite{Bronstein2021,BlumSmithVillar2023,Villar2021}.  Our question is complementary.  When part of the variability of a family is generated by a known group action, exact invariants remove that redundancy before learning, and the exactness can then audit what learned representations preserve.  The machinery is classical geometric invariant theory~\cite{Mumford1994,Dolgachev2003} in its explicit genus-two form: the Igusa invariants coordinatize the moduli of genus-two curves as a weighted projective variety~\cite{Igusa1960}, and the Kummer surface of a Jacobian has an explicit quartic model with coefficients polynomial in the curve coefficients~\cite{CasselsFlynn1996}.

The paper is organized around what the main theorem needs.  \Cref{sec:framework} formalizes canonical algebraic shape spaces and the components a quotient representation requires.  \Cref{sec:kummer} constructs the Kummer quotient via Igusa invariants and describes the dataset of exact equivalence classes it yields.  \Cref{sec:fresnel} derives the intrinsic material chart and the closed-form singular geometry, and identifies the wave surface with the Kummer surface of an explicit split Jacobian.  \Cref{sec:solver} develops the certified interface solve and proves the main theorem.  \Cref{sec:system} describes the linked interactive system.  \Cref{sec:learning} uses the exact equivalence classes to audit learned models.  \Cref{sec:experiments} evaluates both halves of the construction.

\section{Canonical Algebraic Shape Spaces}
\label{sec:framework}

Let $\A$ be a parameter space, $\Phi:\A\to\Sspace$, $a\mapsto S_a$, a map to geometric objects, and $G$ a group acting on $\A$.

\begin{defn}[Task-aware equivalence]
For a visual-computing task $\mathcal{T}$, write $a\sim_{\mathcal{T}} b$ if $b=g\cdot a$ for some $g$ in a designated subgroup $G_{\mathcal T}\le G$ of transformations that the task regards as changes of description rather than changes of state.  The associated canonical shape space is the quotient $\M_{\mathcal T}=\A/G_{\mathcal T}$.
\end{defn}

The subgroup matters.  We distinguish $G_{\rm rep}\subseteq G_{\rm task}\subseteq G_{\rm ambient}$: transformations that change only the representation (rescaling an implicit equation), the full set of nuisance transformations the application removes, and the larger ambient algebraic symmetry group.  Rotating a crystal, for instance, lies in $G_{\rm ambient}$ but not in $G_{\rm task}$ when the illumination frame is fixed; a quotient that removed it would erase state the renderer needs.  Task-aware quotienting removes exactly $G_{\rm task}$ and retains everything else as explicit variables.

The first component of the representation is an invariant encoding.  Let $I_1,\dots,I_m$ be polynomials of weights $w_1,\dots,w_m$ that are $G_{\rm task}$-invariant, or relatively invariant with a common character, $I_j(g\cdot a)=\chi(g)^{w_j}I_j(a)$, and set $E(a)=[\,I_1(a):\cdots:I_m(a)\,]\in\PP(w_1,\dots,w_m).$ Then $E(g\cdot a)=E(a)$ for all $g\in G_{\rm task}$: absolute invariance gives equality componentwise, and relative invariance changes the tuple by a weighted scaling, which is trivial in the weighted projective space.  For reductive actions, finitely many generating invariants exist and separate closed orbits on the stable locus~\cite{Mumford1994,Dolgachev2003}, so a finite invariant vector suffices in the regime where we compute.  Two hypotheses are implicit.  The encoding is defined on the locus where the invariant tuple is nonzero, and the weighted scaling that the projectivization removes must itself belong to $G_{\rm task}$; otherwise $E$ is invariant but not separating, and an affine invariant vector is used instead.  Computationally, orbit invariance means that sampling along a nuisance orbit creates no new quotient states---the property that raw parameterizations lack.

Invariants need not separate every orbit globally, and the quotient is best handled as a stratified object: a stable region where $E$ separates the classes of interest; symmetry strata where stabilizers jump; and a discriminant where the geometry degenerates.  We treat these loci as first-class citizens---event surfaces of the shape space that the interactive system of \cref{sec:system} draws explicitly---rather than as pathologies to be hidden by a chart.

Comparison of quotient points requires a metric.  On a weighted projective quotient with coordinates $x=[x_1:\cdots:x_m]$ of weights $w_j$ we use the weighted chordal distance
\begin{equation}
\label{eq:wdist}
d_{\M}(x,y)=\bigl\|\widehat{x}-\widehat{y}\bigr\|,\qquad \widehat{x}_j=\frac{x_j}{\|x\|_w^{\,w_j}},\qquad \|x\|_w=\Bigl(\textstyle\sum_j |x_j|^{2/w_j}\Bigr)^{1/2},
\end{equation}
which normalizes representatives to the weighted unit sphere before comparison.  We use $d_\M$ as a \emph{computational chart metric} on the real chart in which the encoder is evaluated: the normalization is invariant under the positive weighted rescalings that arise there, but we do not claim an intrinsic distance on the weighted projective quotient in general, and all statements about $d_\M$ are made in this chart.  Two claims must be kept separate: that the quotient removes representational redundancy is exact and algebraic; that a quotient metric correlates with visual or physical similarity is empirical, and the present paper does not evaluate it, and the retrieval and sampling operations below are supported uses rather than evaluated ones.

Finally, since the quotient map is many-to-one, downstream computation needs a representative map $D:\M\dashrightarrow\A$, defined chart by chart, with $S_m=\Phi(D(m))$.  Near singular strata no globally smooth section exists, and the representative map exposes chart changes instead of pretending the quotient is a Euclidean latent vector.

These components support four operations, of which only interpolation is evaluated below (E1): canonical retrieval (encode, then search under $d_\M$), quotient-aware sampling ($\eta$-nets in $\M$ rather than $\A$), representative-independent interpolation (paths in $\M$ decoded through $D$), and stratum-aware navigation (maintaining distances to marked discriminant and symmetry loci so that qualitative events are approached or crossed deliberately).

\section{The Kummer Family and Its Quotient}
\label{sec:kummer}

Kummer quartics are the controlled family on which the framework is instantiated.  They offer an explicit algebraic representation, a genuinely nontrivial moduli structure of known dimension three, and visible singular and symmetry phenomena: a generic Kummer quartic has sixteen nodes over the algebraic closure, arranged in the classical $(16,6)$ configuration, and arises as the quotient $A/\{\pm1\}$ of a principally polarized abelian surface.  The family plays two roles for us---as a controlled quotient whose equivalence classes are known independently of any renderer, and as the ambient family containing the physical Fresnel surfaces of \cref{sec:fresnel}.

A smooth genus-two curve $C:\; y^2=f(x)=\sum_{i=0}^{6}c_i x^i$, with $\deg f\in\{5,6\}$ and $\disc(f)\neq0$, determines the Jacobian $\Jac(C)$, and the quotient $\Jac(C)/\{\pm1\}$ embeds in $\PP^3$ as a quartic $\Kum_C$ whose defining equation has coefficients polynomial in $c_0,\dots,c_6$, written out in full by Cassels and Flynn~\cite[Ch.~3]{CasselsFlynn1996}.  For the Kummer quartics embedded by the linear system $|2\Theta|$, projective equivalence of the polarized embedded models corresponds to isomorphism of the underlying principally polarized Jacobians, and the Torelli theorem then recovers the genus-two curve; equivalently, two curves yield equivalent polarized Kummer models exactly when related by the $\GL_2$ action on the binary sextic $f$.

The encoder of \cref{sec:framework} is realized by the Igusa invariants $J_2,J_4,J_6,J_{10}$, homogeneous of the indicated weights under the $\GL_2$ action, with $J_{10}$ a nonzero scalar multiple of the discriminant of the sextic~\cite{Igusa1960}:
\begin{equation}
\label{eq:igusa-encoder}
E_K(C)=[\,J_2:J_4:J_6:J_{10}\,]\in\PP(2,4,6,10).
\end{equation}
We write $\M_K$ for the weighted projective space $\PP(2,4,6,10)$ and $\pi_K$ for the map sending a Kummer quartic $\Kum_C$, given in Cassels--Flynn form, to $E_K(C)$.  Over an algebraically closed field the moduli space of smooth genus-two curves is the open locus $J_{10}\neq0$; on it, $E_K$ separates isomorphism classes, which is precisely the stable regime of the framework.  Over $\R$ distinct real forms share a moduli point, which is the field-of-definition subtlety addressed below.  The vanishing locus $J_{10}=0$ is the discriminant boundary in the weighted-projective compactification (it lies outside the moduli of smooth curves), and the classical loci of curves with extra automorphisms furnish the symmetry strata; both are recorded by the dataset generator.  For $d_\M$ we use the weighted chordal distance~\eqref{eq:wdist} with weights $(2,4,6,10)$.  A representative map $D_K$ on the generic stratum is given by Mestre's reconstruction of a sextic from a point of moduli~\cite{Mestre1991}, composed with the Cassels--Flynn quartic; near symmetry strata the reconstruction is chart-dependent, and the system reports the chart in use.  Because real points of moduli involve field-of-definition subtleties, the dataset generator samples explicit real sextics first and computes their invariants, rather than sampling weighted-projective points and promising a real reconstruction.

This construction specifies, as a by-product, a shape dataset whose ground truth no learned pipeline can currently provide for itself.  We sample a compact stable subset of $\M_K$ with stratified densification near the marked loci, and for each sampled moduli point the dataset records the quotient coordinates and chart metadata, one chart-normalized representative, stabilizer information, distance to the marked strata, several known-equivalent raw representatives generated by explicit $\GL_2$ and projective transformations, and renderable geometry (implicit equation, mesh, point cloud, rendered multi-view images, and SDF samples).  Equivalence in this dataset is known by construction rather than inferred from approximate geometric matching.  The arithmetic part of such a dataset exists: the database of~\cite{ShaskaShaska2025} records normalized moduli points of $\PP(2,4,6,10)$ with their weighted heights, fine or coarse status, automorphism groups, and membership in $L_3$, $L_5$, $L_7$; the dataset described here attaches the Cassels--Flynn quartic and renderable geometry to each such point.  Using it to audit learned shape encoders such as~\cite{Park2019DeepSDF,Mescheder2019Occupancy} requires a task definition: membership in $L_n$ is a property of the moduli point, and the classifiers of~\cite{ShaskaShaska2025}, trained on Igusa invariants, are the natural first subject of such an audit, the question being whether their output is invariant under the weighted rescaling $\lambda\star p$; in general a projective transformation changes the Euclidean embedded shape, so a Euclidean visual encoder should not be asked to collapse it, while after canonicalization known-equivalent inputs become identical and collapse is trivial; the informative audit lies between these extremes and must itself be defined against the task-aware principle of \cref{sec:framework}.  \Cref{sec:learning} carries out two such audits on the Fresnel slice of \cref{sec:fresnel}, where the classes, the invariants, and the event loci are available in closed form.

\section{Fresnel Wave Surfaces as a Physical Slice}
\label{sec:fresnel}

\subsection{Dispersion and the Fresnel quartic}
\label{sec:dispersion}

Consider a homogeneous, transparent, nonconducting, nonmagnetic ($\mu_r=1$) anisotropic medium at angular frequency $\omega>0$ with relative permittivity tensor $\varepsilon$.  Plane waves $E(x,t)=E_0e^{i(k\cdot x-\omega t)}$ satisfy $M(k)E_0=0$ with $M(k)=kk^\top-\|k\|^2I+k_0^2\varepsilon,\qquad k_0=\omega/c.$ In principal coordinates $\varepsilon=\diag(\varepsilon_1,\varepsilon_2,\varepsilon_3)$, $\varepsilon_i>0$, the determinant factors as $\det M(k)=k_0^2\,F_{\varepsilon,\omega}(k)$ with the \emph{Fresnel polynomial}
\begin{equation}
\label{eq:fresnel}
\begin{split}
F_{\varepsilon,\omega}(k)={}&\|k\|^{2}\bigl(\varepsilon_1k_1^{2}+\varepsilon_2k_2^{2}+\varepsilon_3k_3^{2}\bigr)-k_0^{2}\Bigl[\varepsilon_1(\varepsilon_2{+}\varepsilon_3)k_1^{2}+\varepsilon_2(\varepsilon_3{+}\varepsilon_1)k_2^{2}\\
&+\varepsilon_3(\varepsilon_1{+}\varepsilon_2)k_3^{2}\Bigr]+k_0^{4}\varepsilon_1\varepsilon_2\varepsilon_3,
\end{split}
\end{equation}
whose real zero locus is the two-sheeted quartic Fresnel surface containing all propagating wave vectors~\cite{BornWolf1999}.  Throughout, ``Fresnel surface'' means this fixed-frequency wave-vector (normal) surface, not its ray-surface dual; the distinction matters for conical refraction below.  Two factorizations, obtained by direct expansion against~\eqref{eq:fresnel}, anchor the degenerate regimes.

\begin{lem}[Isotropic and uniaxial factorizations]
\label{lem:factorizations}
If $\varepsilon_1=\varepsilon_2=\varepsilon_3=\epsilon$ then $F=\epsilon(\|k\|^2-k_0^2\epsilon)^2$, a doubled sphere.  If $\varepsilon_1=\varepsilon_2=\epsilon_o$ and $\varepsilon_3=\epsilon_e$ then $F=\bigl(\|k\|^{2}-k_0^{2}\epsilon_o\bigr)\bigl(\epsilon_o(k_1^{2}+k_2^{2})+\epsilon_ek_3^{2}-k_0^{2}\epsilon_o\epsilon_e\bigr),$ the ordinary sphere and extraordinary ellipsoid of a uniaxial medium.
\end{lem}

\Cref{lem:factorizations} pins the uniaxial boundary of the family to the closed-form theory of~\cite{WeidlichWilkie2008} and supplies the renderer with an exact regression target on that stratum.

\subsection{Singular points and optic axes}
\label{sec:singular}

Throughout this subsection assume the strictly biaxial ordering $0<\varepsilon_1<\varepsilon_2<\varepsilon_3$.

\begin{prop}[Singular points in closed form]
\label{prop:nodes}
The real singular points of $\{F_{\varepsilon,\omega}=0\}$ are exactly the four points
\begin{equation}
\label{eq:nodes}
k^{*}=\pm\,k_0\left(\pm\sqrt{\frac{\varepsilon_3(\varepsilon_2-\varepsilon_1)}{\varepsilon_3-\varepsilon_1}},\;0,\;\sqrt{\frac{\varepsilon_1(\varepsilon_3-\varepsilon_2)}{\varepsilon_3-\varepsilon_1}}\right),
\end{equation}
all lying on the sphere $\|k\|=k_0\sqrt{\varepsilon_2}$ in the principal plane spanned by the extreme axes.  The two lines through opposite pairs are the optic axes, and their angle $\beta$ from the $k_3$-axis satisfies
$\tan^{2}\beta=\varepsilon_3(\varepsilon_2-\varepsilon_1)\big/\bigl(\varepsilon_1(\varepsilon_3-\varepsilon_2)\bigr)$, in agreement with the classical formula~\cite{BornWolf1999}.
\end{prop}

\begin{proof}
For $i\in\{1,2,3\}$, differentiating~\eqref{eq:fresnel} gives $\partial_{k_i}F=2k_i\,X_i,\qquad X_i=B(k)+\varepsilon_iA(k)-k_0^{2}\varepsilon_i(\varepsilon_j+\varepsilon_l),$ with $A=\|k\|^{2}$, $B=k^{\top}\varepsilon k$, and $\{i,j,l\}=\{1,2,3\}$, so a singular point of the surface satisfies $F=0$ and, for each $i$, $k_i=0$ or $X_i=0$.  If no coordinate vanishes, then $X_1=X_2=X_3=0$, and the differences factor as $X_1-X_2=(\varepsilon_1-\varepsilon_2)\bigl(A-k_0^{2}\varepsilon_3\bigr),\qquad X_2-X_3=(\varepsilon_2-\varepsilon_3)\bigl(A-k_0^{2}\varepsilon_1\bigr),$ forcing $A=k_0^{2}\varepsilon_3$ and $A=k_0^{2}\varepsilon_1$ simultaneously, impossible for $\varepsilon_1\neq\varepsilon_3$.  If exactly one coordinate vanishes and it is $k_1$, the remaining conditions $X_2=X_3=0$ give $A=k_0^{2}\varepsilon_1$ and $B=k_0^{2}\varepsilon_2\varepsilon_3$; solving with $A=k_2^{2}+k_3^{2}$ and $B=\varepsilon_2k_2^{2}+\varepsilon_3k_3^{2}$ yields $k_2^{2}=k_0^{2}\varepsilon_3(\varepsilon_1-\varepsilon_2)/(\varepsilon_3-\varepsilon_2)<0$, so no real point exists; for $k_3=0$, symmetrically $k_2^{2}=k_0^{2}\varepsilon_1(\varepsilon_2-\varepsilon_3)/(\varepsilon_2-\varepsilon_1)<0$.  If two coordinates vanish, say the point lies on the $k_3$-axis, then $X_3=0$ forces $k_3^{2}=k_0^{2}(\varepsilon_1+\varepsilon_2)/2$, while $F$ restricted to the axis factors as $\varepsilon_3(k_3^{2}-k_0^{2}\varepsilon_1)(k_3^{2}-k_0^{2}\varepsilon_2)$, which is strictly negative at that radius since it lies between the two roots, so no surface point exists there; the other two axes are identical up to relabeling, and finally $F(0)=k_0^{4}\varepsilon_1\varepsilon_2\varepsilon_3\neq0$.  Every real singular point therefore lies in the plane $k_2=0$ with $k_1k_3\neq0$.

In the plane $k_2=0$ direct expansion factors the section into a circle and an ellipse, $F(k_1,0,k_3)=\bigl(k_1^{2}+k_3^{2}-k_0^{2}\varepsilon_2\bigr)\bigl(\varepsilon_1k_1^{2}+\varepsilon_3k_3^{2}-k_0^{2}\varepsilon_1\varepsilon_3\bigr),$ and the singular points of the section are the common zeros of the two factors.  Under the biaxial ordering the circle radius $k_0\sqrt{\varepsilon_2}$ lies strictly between the ellipse semi-axes $k_0\sqrt{\varepsilon_1}$ (along $k_3$) and $k_0\sqrt{\varepsilon_3}$ (along $k_1$), so the two conics cross transversally in exactly four real points; solving the two quadratics simultaneously gives $k_1^{2}=k_0^{2}\,\varepsilon_3(\varepsilon_2-\varepsilon_1)/(\varepsilon_3-\varepsilon_1)$ and $k_3^{2}=k_0^{2}\,\varepsilon_1(\varepsilon_3-\varepsilon_2)/(\varepsilon_3-\varepsilon_1)$, i.e.~\eqref{eq:nodes}, which combined with the first paragraph shows that the real singular locus is exactly these four points, in agreement with the classical enumeration of the tetrahedroid's sixteen nodes, of which exactly four are real~\cite{Hudson1905,FavaroHehl2016}.  Each point is an ordinary node of the surface (not merely of the plane section): a direct computation gives $\det\Hess F(k^{*})=32\,\varepsilon_1\varepsilon_3 k_0^{6}(\varepsilon_1-\varepsilon_2)^{2}(\varepsilon_2-\varepsilon_3)^{2}\neq0$ in the strictly biaxial regime, so the singularity is nondegenerate quadratic.  The angle formula follows from $\tan^{2}\beta=k_1^{*2}/k_3^{*2}$.
\end{proof}

The same case analysis, run over $\C$ and at infinity, exhibits the full singular configuration of the projective surface in closed form.

\begin{prop}[All sixteen singular points]
\label{prop:sixteen}
The Fresnel polynomial~\eqref{eq:fresnel} is homogeneous of degree four in $(k_1,k_2,k_3,k_0)$ and defines a quartic surface $S\subset\PP^{3}$.  Over $\C$ the forms $A=k_1^{2}+k_2^{2}+k_3^{2}$ and $B=k^{\top}\varepsilon k$ are the complexified bilinear quadratic forms, not Hermitian norms.  In the strictly biaxial regime the singular locus of $S$ consists of exactly sixteen points, and each is an ordinary node:
\begin{enumerate}[label=\textup{(\alph*)},nosep]
\item the four real nodes~\eqref{eq:nodes} in the plane $k_2=0$;
\item four points in the plane $k_1=0$, with
\[
k_2^{2}=k_0^{2}\,\frac{\varepsilon_3(\varepsilon_1-\varepsilon_2)}{\varepsilon_3-\varepsilon_2}<0,\qquad k_3^{2}=k_0^{2}\,\frac{\varepsilon_2(\varepsilon_3-\varepsilon_1)}{\varepsilon_3-\varepsilon_2}>0,
\]
and four in the plane $k_3=0$, with
\[
k_2^{2}=k_0^{2}\,\frac{\varepsilon_1(\varepsilon_2-\varepsilon_3)}{\varepsilon_2-\varepsilon_1}<0,\qquad k_1^{2}=k_0^{2}\,\frac{\varepsilon_2(\varepsilon_3-\varepsilon_1)}{\varepsilon_2-\varepsilon_1}>0,
\]
each quadruple forming two complex-conjugate pairs;
\item four points in the plane $k_0=0$, the transversal intersection of the conics $\|k\|^{2}=0$ and $k^{\top}\varepsilon k=0$.
\end{enumerate}
\end{prop}

\begin{proof}
The affine case analysis in the proof of \cref{prop:nodes} is algebraic and remains valid over $\C$.  The branch with no vanishing coordinate is contradictory over any field, since $\varepsilon_1\neq\varepsilon_3$.  The branch $k_1=0$ gives $A=k_0^{2}\varepsilon_1$ and $B=k_0^{2}\varepsilon_2\varepsilon_3$, whose simultaneous solutions are the stated squares, now admissible over $\C$; substituting them into $Q=\varepsilon_2(\varepsilon_3+\varepsilon_1)k_2^{2}+\varepsilon_3(\varepsilon_1+\varepsilon_2)k_3^{2}$ gives $Q=2k_0^{2}\varepsilon_1\varepsilon_2\varepsilon_3$, hence $F=AB-k_0^{2}Q+k_0^{4}\varepsilon_1\varepsilon_2\varepsilon_3=k_0^{4}\varepsilon_1\varepsilon_2\varepsilon_3\,(1-2+1)=0$, so all four points lie on the surface; the branch $k_3=0$ is the same computation with $A=k_0^{2}\varepsilon_3$ and $B=k_0^{2}\varepsilon_1\varepsilon_2$.  The axis and origin cases were excluded by the nonvanishing of $F$ there ($F=-\varepsilon_3k_0^{4}(\varepsilon_1-\varepsilon_2)^{2}/4$ on the $k_3$-axis at the critical radius, and $F(0)\neq0$), which holds over $\C$ as well.  The affine singular locus therefore consists of the twelve points of (a) and (b).

At infinity the quartic restricts to $F(k,0)=\|k\|^{2}\,(k^{\top}\varepsilon k)=A\,B$, and $\partial_{k_0}F=-2k_0Q+4k_0^{3}\varepsilon_1\varepsilon_2\varepsilon_3$ vanishes identically at $k_0=0$, so a point at infinity is singular exactly when $\nabla_k(AB)=2Bk+2A\,\varepsilon k=0$.  If $A=0\neq B$ this forces $k=0$; if $B=0\neq A$ it forces $\varepsilon k=0$, hence $k=0$ since $\varepsilon$ is invertible; if $AB\neq0$ then $\varepsilon k=-(B/A)k$, so $k$ is an eigenvector of $\varepsilon$, on which $B=\varepsilon_iA$, forcing $\varepsilon_iA=-\varepsilon_iA$ and hence $A=0$, a contradiction.  Thus $A=B=0$.  The two conics are smooth and meet transversally, since proportional gradients $k\propto\varepsilon k$ would again make $k$ an eigenvector, on which $A\neq0$; by B\'ezout they meet in exactly four points, giving (c).  The quartic is invariant under the sign changes of $k_1,k_2,k_3$, so one Hessian in each quadruple suffices.  In the affine chart $k_0=1$,
\[
\begin{split}
\det\Hess F &=32\,\varepsilon_2\varepsilon_3(\varepsilon_1-\varepsilon_2)^{2}(\varepsilon_1-\varepsilon_3)^{2}\quad(k_1=0),\\
\det\Hess F &=32\,\varepsilon_1\varepsilon_2(\varepsilon_1-\varepsilon_3)^{2}(\varepsilon_2-\varepsilon_3)^{2}\quad(k_3=0),
\end{split}
\]
and in the affine chart $k_3=1$ one has $\det\Hess F=32(\varepsilon_1-\varepsilon_3)^{2}(\varepsilon_2-\varepsilon_3)^{2}$ at the points of (c).  These values are nonzero, so together with \cref{prop:nodes} every singular point is an ordinary node.  The total of sixteen agrees with the classical enumeration for tetrahedroids~\cite{Hudson1905,FavaroHehl2016}.
\end{proof}

Physically, the four conoidal nodes are the optic-axis wave vectors.  The Fresnel surface carries two distinct, dual singular structures: the four conoidal points with their tangent cones, and four special tangent planes touching the surface along circles~\cite{Rowe2022,Hudson1905}; in the Kummer picture these are four of the nodes and four of the trope planes of the $(16,6)$ configuration.  Both structures participate in conical refraction; since the customary labels ``internal'' and ``external'' depend on whether the wave surface or its ray-surface dual is taken as primary, we refer to the geometry directly and keep the two structures separate throughout.  Computationally, \eqref{eq:nodes} lets the system draw the optic axes without any numerical singularity search (\cref{fig:pipeline}(b)) and predicts where the interface polynomial of \cref{sec:solver} develops multiple roots.

\subsection{The intrinsic material chart}
\label{sec:chart}

Not every algebraic equivalence is a physical equivalence, and the material state must be factored accordingly.  Permuting principal axes is a rotation (up to a fixed reflection convention) and is absorbed into the orientation; overall scaling $\varepsilon\mapsto t\varepsilon$ rescales the wave surface and is retained as an optical scale.  We therefore write the state as
\begin{equation}
\label{eq:state}
q=(m,\rho,[R],\lambda),
\end{equation}
with scale $\rho=\varepsilon_2$ (the middle principal value, after imposing $\varepsilon_1\le\varepsilon_2\le\varepsilon_3$), wavelength $\lambda$ when dispersion is modeled (for dispersive materials $m$ and $\rho$ are functions $m(\lambda),\rho(\lambda)$ of stored dispersion parameters), an orientation class $[R]$ made precise below, and the intrinsic coordinate
\begin{equation}
\label{eq:uv}
m=(u,v)=\Bigl(\frac{\varepsilon_1}{\varepsilon_2},\frac{\varepsilon_3}{\varepsilon_2}\Bigr)\in\Delta=\{(u,v):0<u\le1\le v\}.
\end{equation}
The strictly biaxial materials fill the open cell $u<1<v$; the edges $u=1$ and $v=1$ are the two uniaxial strata, and the corner $u=v=1$ is the isotropic point.  The orientation coordinate is itself stratified, because the eigenframe of a symmetric tensor with ordered distinct eigenvalues is defined only up to the sign flips of the Klein four-group $D_2\subset\SO(3)$: the orientation fiber is $\SO(3)/D_2$ over the biaxial cell, collapses to an unoriented axis $\RP^{2}$ on the uniaxial edges, and disappears at the isotropic corner.  Accordingly $[R]$ denotes the class of a rotation in this quotient, and a renderer works with an arbitrary representative.  The state space is thus a stratified fiber space over $\Delta$, not a product; this is the precise sense in which the material state is not a Euclidean parameter vector, and it is the structure the interpolation experiments of \cref{sec:experiments} probe.  In these coordinates (\cref{fig:pipeline}(a)) the optic-axis angle of \cref{prop:nodes} becomes $\tan^{2}\beta=\frac{v(1-u)}{u(v-1)}.$ Since optic axes are unoriented lines, the geometrically meaningful separation is the line angle $\theta_{\rm OA}=\min(2\beta,\pi-2\beta)\in[0,\pi/2]$, and both degenerations become transparent: $\theta_{\rm OA}\to0$ on either uniaxial edge, the two optic axes coinciding with the distinguished axis of the respective uniaxial limit.  There is no joint limit at the isotropic corner: along $u=1-at$, $v=1+bt$, $t\downarrow0$, one has $\tan^{2}\beta\to a/b$, and an isotropic medium has no distinguished optic axes.  This factorization is task-aware quotienting in action: projective and coordinate redundancy is gone, while orientation and physical scale remain explicit because the renderer needs them.

The chart is not merely a convenient normalization; it is a complete invariant of wave-surface shape, a fact that enters the main theorem of \cref{sec:solver}.

\begin{prop}[The chart classifies wave surfaces]
\label{prop:chart}
Two strictly biaxial permittivities $\varepsilon,\varepsilon'$ have wave surfaces related by a rotation and a uniform rescaling of $k$ if and only if $m(\varepsilon)=m(\varepsilon')$.
\end{prop}

\begin{proof}
Direct substitution in~\eqref{eq:fresnel} gives the equivariances
\[
F_{R\varepsilon R^\top\!,\,\omega}(Rk)=F_{\varepsilon,\omega}(k)\quad(R\in\SO(3)),\qquad F_{t\varepsilon,\omega}\bigl(\sqrt{t}\,k\bigr)=t^{3}F_{\varepsilon,\omega}(k)\quad(t>0),
\]
the second because $A(\sqrt t\,k)=tA(k)$, $B_{t\varepsilon}(\sqrt t\,k)=t^{2}B_\varepsilon(k)$, and the matrix $\Gamma_\varepsilon=\varepsilon(\tr(\varepsilon)I-\varepsilon)$ of \cref{sec:reduction} satisfies $\Gamma_{t\varepsilon}=t^{2}\Gamma_\varepsilon$, so all three terms scale by $t^{3}$.  If $m(\varepsilon)=m(\varepsilon')$ then $\varepsilon'=t\,R\varepsilon R^\top$ for some $t>0$ and $R\in\SO(3)$ (the common ratios fix the eigenvalues up to a common factor, and the ordered eigenframes differ by a rotation), so the surfaces are congruent up to the rescaling $k\mapsto\sqrt{t}\,k$.  Conversely, the three principal radii are metric features of the surface itself: by \cref{prop:nodes} the four nodes lie at radius $k_0\sqrt{\varepsilon_2}$ and span a plane, and the section of the surface by that plane factors into the circle of radius $k_0\sqrt{\varepsilon_2}$ and the ellipse with semi-axes $k_0\sqrt{\varepsilon_1}$ and $k_0\sqrt{\varepsilon_3}$.  A rotation preserves these radii and a rescaling $k\mapsto ck$ multiplies all three by $c$, so the ratios $(u,v)=(\varepsilon_1/\varepsilon_2,\varepsilon_3/\varepsilon_2)$ are invariant, and congruent surfaces have equal charts.
\end{proof}

\subsection{The Kummer moduli of the wave surface}
\label{sec:embedding}

Write $S_\varepsilon\subset\PP^3$ for the projective quartic of \cref{prop:sixteen}.  Its singular locus is finite by \cref{prop:sixteen}.  A reducible or nonreduced quartic is singular along a curve, so $S_\varepsilon$ is irreducible and reduced.  A quartic surface with sixteen nodes is the Kummer surface of a principally polarized abelian surface $A$~\cite{Hudson1905,Nikulin1975}, and $A$ is indecomposable, since a product polarization gives a double quadric rather than an irreducible quartic~\cite[Ch.~10]{BirkenhakeLange2004}; by Weil's theorem $A=\Jac(C)$ for a smooth genus-two curve $C$.  We use the Kummer plane of~\cite{ClingherMalmendierShaska2026}.  Projecting $S_\varepsilon$ from a node $o$ gives $\PP^2=\operatorname{Sym}^2(C/w)$, the space of unordered pairs of $x$-coordinates.  The six tropes through $o$ project to the six \emph{Kummer lines} $\ell_r=\{\{r,x\}:x\in\PP^1\}$ indexed by the six branch points $r$ of $C$.  The fifteen other nodes project to the fifteen points $\ell_r\cap\ell_{r'}$.  The tangent cone at $o$ projects to the \emph{node conic} $E$, the diagonal $\{\{x,x\}\}$, and $\ell_r$ is tangent to $E$ at $\{r,r\}$~\cite{Hudson1905}, \cite[Def.~3.1, Lem.~3.2]{ClingherMalmendierShaska2026}.  Let $\M_{\rm phys}\subset\M_2$ be the image of the open cell $u<1<v$ of $\Delta$ under $(u,v)\mapsto[C_{u,v}]$, where $C_{u,v}:y^2=(x^2-u)(x^2-1)(x^2-v)$.  Scaling $x$ identifies $C_\varepsilon$ below with $C_{u,v}$ for $(u,v)=m(\varepsilon)$.

\begin{thm}[Kummer moduli of the wave surface]
\label{thm:kummer}
Let $0<\varepsilon_1<\varepsilon_2<\varepsilon_3$, let $\beta$ be the optic-axis half-angle of \cref{prop:nodes}, and put $C_\varepsilon:\ y^2=(x^2-\varepsilon_1)(x^2-\varepsilon_2)(x^2-\varepsilon_3).$ \begin{enumerate}[label=(\roman*),nosep]
\item Over $\C$, $S_\varepsilon$ is projectively equivalent to the Kummer surface of $\Jac(C_\varepsilon)$.  The equivalence is not defined over $\R$: all sixteen nodes of the Kummer surface of $\Jac(C_\varepsilon)$ are real, since the six branch points are real, while $S_\varepsilon$ has four real nodes.  The real structure of $S_\varepsilon$ is that of a twist of $\Jac(C_\varepsilon)$ whose real two-torsion has order four.
\item $\Jac(C_\varepsilon)$ is $(2,2)$-isogenous to $E_1\times E_2$, where $E_1:\ y^2=(x-\varepsilon_1)(x-\varepsilon_2)(x-\varepsilon_3),\qquad E_2:\ y^2=x(x-\varepsilon_1)(x-\varepsilon_2)(x-\varepsilon_3),$ with Legendre parameters $\lambda_1=(\varepsilon_2-\varepsilon_1)/(\varepsilon_3-\varepsilon_1)=(1-u)/(v-u)$ and $\lambda_2=\sin^2\beta=v(1-u)/(v-u)$.  Hence $\M_{\rm phys}$ lies in the Humbert surface $H_4$~\cite{Humbert1899,Kani1994}, and it is Zariski dense in $L_2=\iota^{-1}(H_4)$, where $\iota:\M_2\to\mathcal A_2$ is the Torelli map, the locus of genus-two curves with an extra involution~\cite{ShaskaVolklein2004}.
\item The circle and the ellipse of \cref{prop:nodes} are the images in $S_\varepsilon$ of $E_2$ and $E_1$ respectively.  Under the projection from a real node, the optic-axis plane $k_2=0$ maps to the common shadow line of the two degree-two elliptic subcovers, and its two points on the node conic are the tangent directions at the node of the circle and of the ellipse.
\item Write $a=\sqrt u$ and $c=\sqrt v$.  Let $V_4\cong\Z/2\Z\times\Z/2\Z$ be the Klein four-group, and let $D_n$ be the dihedral group of order $2n$.  Then $\operatorname{Aut}_\C(C_\varepsilon)\cong V_4$ except on two curves: on $uv=1$, that is $\varepsilon_2^{2}=\varepsilon_1\varepsilon_3$, it is $D_4$ of order $8$; on $(a+c)(1+ac)=a^{2}+c^{2}-6ac$ it is $D_6$ of order $12$; and at their unique intersection $(u,v)=(7-4\sqrt3,\,7+4\sqrt3)$ it is the group of order $24$ of $y^{2}=x^{6}-1$, and there $\theta_{\rm OA}=30^\circ$.  No material has $|\operatorname{Aut}(C_\varepsilon)|=48$.  Denote the dihedral invariants of~\cite{ShaskaVolklein2004} by $(\mathfrak s,\mathfrak t)=(\alpha\gamma,\ \alpha^{3}+\gamma^{3})$ to distinguish them from the chart $(u,v)$.  They are defined for the normal form $y^{2}=x^{6}+\alpha x^{4}+\gamma x^{2}+1$.  They are rational functions of $(u,v)$,
\[
\mathfrak s=\frac{(1+u+v)(u+v+uv)}{uv},\qquad \mathfrak t=\frac{(1+u+v)^{3}uv+(u+v+uv)^{3}}{u^{2}v^{2}},
\]
the two curves are $\mathfrak t^{2}=4\mathfrak s^{3}$ and $4\mathfrak t=\mathfrak s^{2}-110\mathfrak s+1125$, and the intersection is $(\mathfrak s,\mathfrak t)=(225,6750)$.
\item Within $\Delta$ the pullback of the discriminant $J_{10}=0$ of $\M_K$ is exactly the two uniaxial edges.  On a generic point of an edge the sextic of $C_\varepsilon$ has two double roots, and at the isotropic corner it has two triple roots.  Further boundary points, such as the limit $u\to0$ with $v$ fixed, arise in a compactification of the physical family.
\end{enumerate}
\end{thm}

\begin{proof}
(i) Work over $\C$ in the affine chart $k_0=1$.  Project from the real node $o=k^{*}=(a,0,b)$ of~\eqref{eq:nodes}, with $a^{2}=\varepsilon_3(\varepsilon_2-\varepsilon_1)/(\varepsilon_3-\varepsilon_1)$ and $b^{2}=\varepsilon_1(\varepsilon_3-\varepsilon_2)/(\varepsilon_3-\varepsilon_1)$, using the direction $(x:y:z)$ of a line through $o$ as coordinate on the Kummer plane.  The other three real nodes lie in the plane $k_2=0$, so they project to the line $L=\{y=0\}$, at $t=x/z\in\{0,\infty,a/b\}$; the twelve complex nodes have $k_2\neq0$ by \cref{prop:sixteen} and do not.  $L$ is not a Kummer line, since $k_2=0$ meets $S_\varepsilon$ in two distinct conics and a trope meets it in a double conic.  In the $(16,6)$ configuration two nodes lie on exactly two common tropes~\cite{Hudson1905}, so each of the three points lies on two Kummer lines, and no Kummer line contains two of them, since it would then equal $L$.  The six Kummer lines are therefore matched in three pairs by the three points, which are $\ell_r\cap\ell_{r'}$ for a perfect matching $\{r,r'\}$ of the six branch points.  In $\operatorname{Sym}^2(\PP^1)$ a line is a pencil of binary quadratics; $L$ has no base point, a base point $r$ making it $\ell_r$, so $L$ is the set of pairs $\{x,\bar\sigma(x)\}$ of a Möbius involution $\bar\sigma$ of $\PP^1=C/w$, with $\bar\sigma(r)=r'$ on the matched pairs.  The map $x\mapsto\{x,\bar\sigma(x)\}$ is the quotient $\PP^1\to L$ by $\bar\sigma$, and $L\cap E$ consists of its two branch points, the fixed points of $\bar\sigma$.  The Hessian of~\eqref{eq:fresnel} at $k^{*}$ has $H_{12}=H_{23}=0$, so $L\cap E$ is the binary quadratic $H_{11}t^{2}+2H_{13}t+H_{33}=0$, whose roots are $t_{\rm circ}=-\frac{b}{a},\qquad t_{\rm ell}=-\frac{\varepsilon_3\,b}{\varepsilon_1\,a},$ the tangent directions at $k^{*}$ of the circle $k_1^{2}+k_3^{2}=k_0^{2}\varepsilon_2$ and of the ellipse $\varepsilon_1k_1^{2}+\varepsilon_3k_3^{2}=k_0^{2}\varepsilon_1\varepsilon_3$.  Using $a^{2}+b^{2}=\varepsilon_2$ and $\varepsilon_1a^{2}+\varepsilon_3b^{2}=\varepsilon_1\varepsilon_3$, the cross-ratios of the labelled triple $(0,\,a/b,\,\infty)$ with $t_{\rm ell}$ and with $t_{\rm circ}$ equal those of $(\varepsilon_1,\varepsilon_2,\varepsilon_3)$ with $\infty$ and with $0$.  A M\"obius transformation is determined by the images of three points.  Hence the unique M\"obius transformation $\mu$ with $\mu(0)=\varepsilon_1$, $\mu(a/b)=\varepsilon_2$, and $\mu(\infty)=\varepsilon_3$ satisfies $\mu(t_{\rm ell})=\infty$ and $\mu(t_{\rm circ})=0$.  The latter is the configuration on the quotient line $\PP^1_{x^{2}}$ of $C_\varepsilon$ by $x\mapsto-x$: the branch points $\pm\sqrt{\varepsilon_i}$ pair to $\varepsilon_i$ and the fixed points to $0,\infty$.  Over $\C$ a double cover of $\PP^1$ branched at two given points is unique, and the branch set of $C$ is the preimage of the triple, so $(C/w,\,\text{branch set})\cong(\PP^1_x,\{\pm\sqrt{\varepsilon_i}\})$ and $C\cong C_\varepsilon$ over $\C$.  The count of real nodes in \cref{prop:sixteen} shows that the real structures differ, as stated.  The Kummer surface of a Jacobian is determined by the curve, which gives (i).

(ii) The involution $\sigma:x\mapsto-x$ of $C_\varepsilon$ has quotients $E_1=C_\varepsilon/\sigma$ and $E_2=C_\varepsilon/\sigma w$ with the stated models, and $E_1\times E_2\to\Jac(C_\varepsilon)$ is an isogeny of type $(2,2)$, so $\Jac(C_\varepsilon)\in H_4$~\cite{Kani1994}, \cite[Prop.~2.10]{ClingherMalmendierShaska2026}.  The Legendre parameter of $E_1$ is the cross-ratio of $(\varepsilon_1,\varepsilon_2;\varepsilon_3,\infty)$, and that of $E_2$ is the cross-ratio of $(0,\varepsilon_1;\varepsilon_2,\varepsilon_3)$, which is $a^{2}/\varepsilon_2=\sin^{2}\beta$ up to the $S_3$ action on Legendre parameters.  The map $(u,v)\mapsto(\lambda_1,\lambda_2)$ has Jacobian determinant $-(1-u)(v-1)/(v-u)^{3}\neq0$ on the open cell, so $\M_{\rm phys}$ is two-dimensional in the irreducible surface $L_2$, hence Zariski dense.

(iii) The image in $S_\varepsilon$ of an elliptic curve $E'\subset\Jac(C_\varepsilon)$ with $(\Theta\cdot E')=2$ is a conic through the four nodes $E'[2]$~\cite[Sec.~3]{ClingherMalmendierShaska2026}.  The four real nodes lie on the circle and on the ellipse, which are the two components of the plane section $k_2=0$, so these are the images of the two elliptic subgroups.  By (i), $L$ is the line of pairs of $\bar\sigma$, and by~\cite[Prop.~4.11]{ClingherMalmendierShaska2026} with $n=2$ it is the shadow of each subcover, the image of the origin lying on $E$; the elliptic curves $E_1\times0$ and $0\times E_2$ meet transversally at the origin, so their images have distinct tangent directions at $o$, which are the two points of $L\cap E$.  The tangent direction $t_{\rm ell}$ gives the parameter $\lambda_1$ of $E_1$, so the ellipse is the image of $E_1$ and the circle that of $E_2$.  The assignment also follows directly.  Let $K$ be the canonical class of $C_\varepsilon$.  The abelian subvariety of $\Jac(C_\varepsilon)$ corresponding to $E_1$ consists of the classes $[P+\sigma P-K]$.  In the Kummer plane its points are the pairs $\{x(P),-x(P)\}$, so its image is $L$.  The class $[P+\sigma P-K]$ is zero if and only if $\sigma P=wP$, that is, $P$ is fixed by $\sigma w$.  The map $\sigma w$ sends $(x,y)$ to $(-x,-y)$ and fixes $y/x^{3}$, so it fixes the two points at infinity.  It has no finite fixed points, since $y^{2}=-\varepsilon_1\varepsilon_2\varepsilon_3\neq0$ at $x=0$.  Hence the image of $E_1$ reaches $o$ in the direction of the point $\{\infty,\infty\}$ of $E$, which is $t_{\rm ell}$ because $\mu(t_{\rm ell})=\infty$.  For $E_2$ the classes are $[P+\sigma wP-K]$, the origin corresponds to the fixed points of $\sigma$, which lie over $x=0$, and the direction is $\{0,0\}$, which is $t_{\rm circ}$.

(iv) A permutation $\pi$ of the branch set $W=\{\pm\sqrt{\varepsilon_i}\}$ is induced by a Möbius transformation if and only if the cross-ratios of $(w_1,w_2,w_3,w_i)$ and $(\pi w_1,\pi w_2,\pi w_3,\pi w_i)$ agree for $i=4,5,6$.  Enumerating the $720$ permutations with $\sqrt{\varepsilon_1}=1$ gives polynomial conditions in $\sqrt{\varepsilon_2},\sqrt{\varepsilon_3}$ whose zero sets in the region $1<\sqrt{\varepsilon_2}<\sqrt{\varepsilon_3}$ are exactly the two stated curves; the computation is in Section~SM3 of the supplementary materials.  The reduced automorphism group has order $2$ generically, $4$ on $uv=1$, $6$ on the second curve, and $12$ at their intersection, so $\operatorname{Aut}(C_\varepsilon)$ has orders $4$, $8$, $12$, $24$, and the groups are $V_4$, $D_4$, $D_6$, and the group of $y^{2}=x^{6}-1$ by the classification of automorphism groups of genus-two curves with an extra involution~\cite{Bolza1887,ShaskaVolklein2004}; the group of order $48$ does not occur.  The dihedral invariants of $C_\varepsilon$ follow from scaling $x$ by a sixth root of $-uv$.  On $uv=1$ one has $\alpha=\gamma$, which gives $\mathfrak t^{2}=4\mathfrak s^{3}$.  The relation $4\mathfrak t=\mathfrak s^{2}-110\mathfrak s+1125$ holds on the second curve by elimination, as shown in Section~SM4 of the supplementary materials.  On $uv=1$ one has $\tan^{2}\beta=v$, and $v=7+4\sqrt3=\tan^{2}75^\circ$ gives $\theta_{\rm OA}=30^\circ$.  On the second curve $\varepsilon_3/\varepsilon_1\ge(7+4\sqrt3)^{2}=97+56\sqrt3$, with equality only at the intersection point.

(v) The discriminant of $(x^{2}-u)(x^{2}-1)(x^{2}-v)$ is $64uv(u-1)^{4}(v-1)^{4}(u-v)^{4}$.  On $\Delta$ it vanishes exactly on $u=1$ and $v=1$.  On $u=1$ with $v>1$ the roots $\pm1$ are double; at $u=v=1$ the roots $\pm1$ are triple.  The limit $u\to0$ gives $y^{2}=x^{2}(x^{2}-1)(x^{2}-v)$, which also lies on $J_{10}=0$.
\end{proof}

By \cref{thm:kummer} the map from the chart to the Igusa coordinates~\eqref{eq:igusa-encoder} is the evaluation of the Igusa invariants on the sextic $(x^{2}-u)(x^{2}-1)(x^{2}-v)$, a computation the dataset generator of \cref{sec:kummer} performs directly; in the chart the Igusa--Clebsch invariants $I_2,I_4,I_6,I_{10}$, defined by the root differences of the sextic as in~\cite{Igusa1960} and of which $J_2,\dots,J_{10}$ are polynomials, satisfy $I_2=16\,(u^{2}v+u^{2}+uv^{2}+18uv+u+v^{2}+v)$ and $I_4,I_6,I_{10}$ have degrees $6,9,14$.  Part (ii) is consistent with the equation of $L_2$.  The locus $L_2$ is the zero set of the degree-$30$ polynomial $J_{30}$ in the Igusa--Clebsch invariants displayed in~\cite{ShaskaShaska2025}.  Since $C_{u,v}$ has the involution $x\mapsto-x$, the polynomial $J_{30}$ vanishes identically on the invariants of $C_{u,v}$.  The main theorem lives entirely in the chart $\Delta$ and does not use the Igusa coordinates.  The moduli map $(u,v)\mapsto[C_{u,v}]$ is not injective, since it identifies $(u,v)$ with $(1/v,1/u)$, so the complex moduli point is a coarser invariant than the chart.  It places the physical family inside a moduli space whose equivalence classes are known independently of any renderer, and it labels the strata of $\Delta$ in the language of \cref{sec:framework}: the two curves of (iv) are the symmetry strata, their intersection the deepest one, and the uniaxial edges the discriminant.  Physically, $E_1$ is the curve of the three permittivities and $E_2$ the curve of the optic-axis angle; with the convention that a crystal is positive when the acute bisectrix of the optic axes is the $\varepsilon_3$-axis, the optic sign is the sign of $\tfrac12-\lambda_2$, and the crystal is neutral exactly when $\varepsilon_2=2\varepsilon_1\varepsilon_3/(\varepsilon_1+\varepsilon_3)$.  The ray surface of $\varepsilon$ is the wave surface of $\varepsilon^{-1}$, which has chart $(1/v,1/u)$; the involution $(u,v)\mapsto(1/v,1/u)$ of $\Delta$ preserves the complex isomorphism class of $C_\varepsilon$, and its fixed curve $uv=1$ is exactly the $D_4$ stratum.  In dimensionless wave-normal and ray coordinates, the media whose wave and ray surfaces are congruent up to rotation and uniform rescaling are, by \cref{prop:chart}, those with $\varepsilon_2^{2}=\varepsilon_1\varepsilon_3$, and there $E_1\cong E_2$ over $\C$.  The intersections $L_2\cap L_n$ for $n\ge3$ may pull back to further curves in $\Delta$; on such a curve $\Jac(C_\varepsilon)$ has an elliptic subcover of degree $n$, and the wave surface carries a rational curve of degree $n$ through four nodes~\cite{ClingherMalmendierShaska2026}.  Which of these real loci are nonempty is not determined here, and we do not pursue them.  For the general Kummer family of \cref{sec:kummer} the loci $L_n$ are the natural arithmetic strata; their equations $F_n\in\Z[J_2,J_4,J_6,J_{10}]$ are known for $n\le5$~\cite{Shaska2001,BruinDoerksen2011,MagaardShaskaVolklein2009}, and dominant families exist for $n\le11$~\cite{Kumar2015}.

Three distinct event types organize everything that follows, and the system keeps them visually separate: singularities intrinsic to a fixed wave surface (the four nodes), degenerations of the material family (the edges and corner of $\Delta$), and ray-level events where the interface polynomial acquires a multiple root.  Conflating them is precisely the kind of unattributed artifact the representation is designed to prevent.

\section{Certified Algebraic Biaxial Refraction}
\label{sec:solver}

\subsection{Reduction to an exact quartic in one variable}
\label{sec:reduction}

At an interface with normal $n$, oriented from the interface into the transmitting medium, the boundary conditions preserve the tangential wave vector $k_\parallel$, so every candidate transmitted mode lies on the line $k(s)=k_\parallel+sn$.  To substitute this line into~\eqref{eq:fresnel}, write $F(k)=A(k)B(k)-k_0^2Q(k)+k_0^4\varepsilon_1\varepsilon_2\varepsilon_3$ with the quadratic forms $A=\|k\|^2$, $B=k^\top\varepsilon k$, and $Q=k^\top\Gamma_\varepsilon k$, where $\Gamma_\varepsilon=\varepsilon(\tr(\varepsilon)I-\varepsilon)=\diag(\varepsilon_1(\varepsilon_2{+}\varepsilon_3),\varepsilon_2(\varepsilon_3{+}\varepsilon_1),\varepsilon_3(\varepsilon_1{+}\varepsilon_2))$ in principal coordinates; for an oriented crystal, $\varepsilon$ is replaced by $R\varepsilon R^\top$ throughout.  The three forms restrict to the line as $A(s)=a_0+2a_1s+a_2s^2,\qquad B(s)=b_0+2b_1s+b_2s^2,\qquad Q(s)=q_0+2q_1s+q_2s^2,$ with $a_0=k_\parallel^\top k_\parallel,\quad a_1=n^\top k_\parallel,\quad a_2=n^\top n,\qquad b_0=k_\parallel^\top\varepsilon k_\parallel,\quad b_1=n^\top\varepsilon k_\parallel,\quad b_2=n^\top\varepsilon n,$ and $q_0,q_1,q_2$ the analogous contractions with $\Gamma_\varepsilon$.  Here $s$ is a line parameter; the normal component of $k(s)$ is $k_\parallel\cdot\hat n+s\|n\|$ with $\hat n=n/\|n\|$.  Multiplying out yields the \emph{interface polynomial}
\begin{equation}
\label{eq:ps}
\begin{split}
p(s)=F_{\varepsilon,\omega}(k_\parallel+sn)={}&a_2b_2 s^{4}+2\bigl(a_1b_2+a_2b_1\bigr)s^{3}+\bigl(a_0b_2+4a_1b_1+a_2b_0-k_0^{2}q_2\bigr)s^{2}\\
&+2\bigl(a_0b_1+a_1b_0-k_0^{2}q_1\bigr)s+a_0b_0-k_0^{2}q_0+k_0^{4}\varepsilon_1\varepsilon_2\varepsilon_3.
\end{split}
\end{equation}
The derivation deliberately assumes neither $\|n\|=1$ nor $k_\parallel\perp n$; when both hold, $a_1=0$ and $a_2=1$ and~\eqref{eq:ps} reduces to the familiar Booker form with leading coefficient $b_2$.  The general form is what makes the exact-arithmetic claims below watertight: a floating-point normal treated as an exact dyadic rational is in general neither exactly unit nor exactly orthogonal to the stored tangential vector, and~\eqref{eq:ps} equals $F_{\varepsilon,\omega}(k_\parallel+sn)$ exactly for the represented inputs regardless.  All quantities entering the solver, including $\disc(p)$, are polynomial with integer coefficients in the entries of $(\varepsilon,n,k_\parallel,k_0)$ and are evaluated in interval or multiprecision arithmetic directly from these formulas; for the classical background see Born and Wolf~\cite{BornWolf1999}.

\begin{lem}[Exact degree]
\label{lem:degree}
For every nonzero normal $n$ and positive-definite $\varepsilon$, the leading coefficient of $p$ is $a_2b_2=(n^\top n)(n^\top\varepsilon n)>0$.  Hence $p$ has degree exactly four for every interface configuration, with no normalization of $n$ required, and never degenerates to a lower-degree solve.
\end{lem}

\begin{proof}
The degree-four homogeneous part of~\eqref{eq:fresnel} is $F_4(k)=\|k\|^2(k^\top\varepsilon k)$, and the coefficient of $s^4$ in $F(k_\parallel+sn)$ is $F_4(n)=\|n\|^{2}\,(n^\top\varepsilon n)$, positive by definiteness for $n\neq0$.
\end{proof}

The real roots of $p$ are candidate bulk wave numbers---in the propagating regime four, two forward and two backward---and the transmitted set is selected afterwards: for each root, a polarization $E_i$ spanning $\ker M(k(s_i))$ is computed, the magnetic field $H_i\propto k_i\times E_i$ follows, and the mode is admitted when the time-averaged Poynting flux $\mathbf S_i\cdot n=\tfrac12\Re\bigl(E_i\times H_i^{*}\bigr)\cdot n$ is positive, $n$ pointing into the transmitting medium; the surface gradient $\nabla_kF$ serves only as a geometric diagnostic.  The kernel of $M(k_i)$ is one-dimensional at every smooth point of the wave surface and two-dimensional exactly at the four nodes, where the polarization is undetermined and the Poynting vector sweeps the cone of conical refraction; the solver detects this case by the exact node predicate of \cref{sec:solver-alg} and reports it as a conical event rather than classifying a single polarization.  In the study material this selection admits exactly two forward modes in every sampled configuration (\cref{sec:experiments}), and \cref{lem:flux} below shows that the selection is decided exactly.  The reduction is algebraically elementary, but it converts a multidimensional nonlinear solve into a structured univariate problem with exact certification tools---which is where the quotient viewpoint pays its way in the renderer.

\subsection{The certified solver}
\label{sec:solver-alg}

\Cref{alg:solver} summarizes the solver.  The inputs are rational; the binary64 values stored by a renderer are dyadic rationals and are a special case.  The coefficients of $p$ are kept as exact rationals, and floating-point interval enclosures of them serve only as filters.  Real roots are counted and isolated by a certified method---Sturm sequences or Descartes/Vincent subdivision~\cite{CollinsAkritas1976,RouillierZimmermann2004}---on an interval given by a root bound, so that the isolating intervals contain all real roots.  They are refined by bisection or interval Newton on the square-free part $p_{\rm sf}$.  The discriminant predicate is staged: a floating-point interval evaluation decides most configurations immediately, the precision is raised through a finite schedule while the interval sign is undecided, and the final stage computes the sign exactly by rational subresultants.  At a branch event the root list is still produced: the square-free factorization of $p$ over $\Q$ gives $p_{\rm sf}$ and the multiplicities.  Each root is returned as its isolating interval together with $p_{\rm sf}$, which is an exact representation of the algebraic number; its numerical value is a separate approximation.

\begin{algorithm}[t]
\caption{Certified biaxial interface solve}
\label{alg:solver}
\begin{algorithmic}[1]
\Require $\varepsilon\succ0$, $\omega>0$, $n\neq0$, $k_\parallel$ with rational entries; finite precision schedule $\beta_0<\dots<\beta_{\max}$
\State compute the exact rational coefficients of $p(s)$ from~\eqref{eq:ps}; keep them
\State $\sigma\gets$ sign of $\disc(p)$ from an interval enclosure at $\beta_0$; while undecided, raise the precision along the schedule; if undecided at $\beta_{\max}$, $\sigma\gets$ exact sign of $\disc(p)$ by rational subresultants
\If{$\sigma\neq0$}
  \State $p_{\rm sf}\gets p$; all multiplicities $m_i\gets1$
\Else
  \State square-free factorization of $p$ over $\Q$; $p_{\rm sf}\gets$ product of the distinct factors; $m_i\gets$ multiplicities; mark \textsc{BranchEvent}
\EndIf
\State $\{I_1<\dots<I_r\}\gets$ certified isolation of all real roots of $p_{\rm sf}$ on $[-B,B]$, $B$ a root bound
\State $h\gets\gcd(p_{\rm sf},G_1,G_2,G_3)$ with $G_j(s)=\partial_{k_j}F(k_\parallel+sn)$; match the real roots of $h$ to the $I_i$; mark those $I_i$ \textsc{ConicalNode} (kernel dimension $2$)
\For{each $I_i$ not marked \textsc{ConicalNode}}
  \If{$m_i>1$} mark \textsc{Grazing} ($\mathbf S_i\cdot n=0$ by \cref{lem:flux})
  \Else\ refine $I_i$ on $p_{\rm sf}$ until $q(s)=p'(s)\,e_2(M(k_\parallel+sn))$ has constant sign on $I_i$; admit iff that sign is negative
  \EndIf
  \State $E_i\gets$ approximate null vector of $M(k_\parallel+\operatorname{mid}(I_i)\,n)$ (numerical, for rendering only)
\EndFor
\State \Return the intervals $I_i$ with $p_{\rm sf}$, multiplicities $m_i$, event flags, admission bits, and the $E_i$
\end{algorithmic}
\end{algorithm}

\begin{prop}[Certification invariant]
\label{prop:certified}
Assume the isolator returns pairwise disjoint intervals $I_1<\dots<I_r$, each certified to contain exactly one real root of $p_{\rm sf}$, together with a certificate that $p_{\rm sf}$ has no real root outside $\bigcup_iI_i$, obtained from a root bound and a complete subdivision of the bounded interval.  Assume interval refinement preserves inclusion.  Then the number and ordering of the distinct real roots of $p$ are certified independently of the convergence behavior of the refinement step, and the multiplicities are attached from the square-free factorization.
\end{prop}

\begin{proof}
The isolation certificate consists of finitely many verified sign conditions (Sturm counts or Descartes bounds on subdivided intervals) evaluated in interval arithmetic; each asserts exactly one root in $I_i$ and none in the gaps between consecutive intervals or in the two tails of $[-B,B]$, and depends only on interval evaluations of $p_{\rm sf}$ and its derivatives, never on an iterate.  The interval Newton operator maps each $I_i$ into a subinterval containing the same root, so the count and ordering persist under refinement.
\end{proof}

A root is multiple exactly when $\disc(p)=(-1)^{6}\Res(p,p')/\lead(p)=0$, and because the inputs are rational the staged predicate always terminates with a decision---provably simple or provably multiple; an unresolved outcome can arise only from an explicitly imposed resource limit, never from the mathematics.  The node predicate is exact as well: $k_\parallel+sn$ is a node exactly when $F$ and $\nabla F$ vanish there, that is, when $s$ is a common root of $p$ and the three cubics $G_j(s)=\partial_{k_j}F(k_\parallel+sn)$, and the real roots of $h=\gcd(p_{\rm sf},G_1,G_2,G_3)$ are the node encounters.

The flux classification is exact for the same reason.

\begin{lem}[Algebraic flux predicate]
\label{lem:flux}
Let $s_0$ be a real root of $p$ such that $k=k_\parallel+s_0n$ is a smooth point of the wave surface, and let $E$ be a real unit vector spanning $\ker M(k)$.  Put $\mathbf S\cdot n=(k\cdot n)-(E\cdot k)(E\cdot n)$, which is the time-averaged Poynting flux through the interface up to a positive factor, and let $e_2(M(k))$ be the sum of the principal $2\times2$ minors of $M(k)$.  Then
\[
k_0^{2}\,p'(s_0)=-2\,e_2\bigl(M(k)\bigr)\,\mathbf S\cdot n,
\]
and $e_2(M(k))\neq0$.  Hence $\mathbf S\cdot n=0$ if and only if $s_0$ is a multiple root of $p$, and at a simple root the mode is admitted if and only if $p'(s_0)\,e_2(M(k))<0$.
\end{lem}

\begin{proof}
At a smooth point $\ker M(k)$ is one-dimensional~\cite{BornWolf1999}, so the symmetric matrix $M(k)$ has eigenvalues $0,\eta_1,\eta_2$ with $\eta_1\eta_2\neq0$, its adjugate is $\eta_1\eta_2\,EE^{\top}$, and $e_2(M(k))=\eta_1\eta_2$.  Since $\det M(k(s))=k_0^{2}p(s)$, Jacobi's formula gives $k_0^{2}p'(s_0)=\tr\bigl(\operatorname{adj}M\cdot M'\bigr)=\eta_1\eta_2\,E^{\top}M'E$ with $M'=nk^{\top}+kn^{\top}-2(k\cdot n)I$.  Hence $k_0^{2}p'(s_0)=2\eta_1\eta_2\bigl((E\cdot n)(E\cdot k)-(k\cdot n)\bigr)$.  The magnetic field is $H\propto k\times E$, and $E\times(k\times E)=k-E(E\cdot k)$ for unit $E$, so the flux through the interface is a positive multiple of $(k\cdot n)-(E\cdot k)(E\cdot n)$.  The two displayed expressions agree.  The last statement follows since $k_0>0$.
\end{proof}

The sign of $p'(s)\,e_2(M(k_\parallel+sn))$, a polynomial in $s$ with rational coefficients, at the algebraic number $s_0$ is decided by refining the isolating interval until the polynomial has constant sign on it; this terminates at every simple root by \cref{lem:flux}.  Three clarifications delimit what is certified.  First, $\disc(p)=0$ is an \emph{algebraic root event}: at a smooth point of the surface it is a tangency of the interface line, at a node it is a singular encounter with a two-dimensional kernel of $M$, and the repeated root may also be complex; reality and the kernel dimension are tested separately.  Second, by \cref{lem:flux} a real mode at a smooth point has zero flux exactly when it is a repeated root, so the transmitted set cannot change without a root collision, and the discriminant is the only event locus.  Third, the inputs are rational numbers, in practice IEEE floating-point numbers treated as exact dyadic rationals, so the certificates apply to the represented computational state, not to uncertainty in a measured physical material.  In summary: the number and ordering of distinct real roots, all repeated-root events, the node encounters, and the flux classification of simple roots are certified; only the polarization vectors used for rendering are numerical approximations.

The solver addresses propagating geometrical-optics transmission directions; a complete Maxwell interface solution additionally carries evanescent complex modes to satisfy boundary conditions, notably near total internal reflection, and we scope those outside the present renderer.  We deliberately do not build the solver on quartic radical formulas.  Radicals exist, but their direct evaluation can be poorly conditioned and is not automatically faster than robust numerics; the contribution is the exact degree, the certified counts, and the event predicates, for which isolation methods are the appropriate tool.

The order at which the discriminant approaches an optic-axis event is itself a consequence of the closed-form singular geometry.

\begin{lem}[Quadratic event order at a node]
\label{lem:order}
Let $k^{*}$ be a node of the wave surface and let $\ell_\delta=\{k_\parallel(\delta)+sn:s\in\R\}$ be a family of interface lines depending smoothly on $\delta\ge0$ with $\dist(\ell_\delta,k^{*})=\delta\,(1+o(1))$ and $k^{*}\in\ell_0$.  Write $H(x,y)=\tfrac12\,x^{\top}\Hess F(k^{*})\,y$, and assume the genericity conditions: $H(n,n)\neq0$; the limiting offset direction $\bar x$ of the nearest point of $\ell_\delta$ to $k^{*}$ satisfies $D(\bar x):=H(\bar x,n)^{2}-H(n,n)H(\bar x,\bar x)\neq0$; and $\ell_0$ meets the surface in two further simple points.  Then $p_\delta$ has a pair of roots---real if $D(\bar x)>0$, complex conjugate if $D(\bar x)<0$---with gap $\Theta(\delta)$, its remaining two roots stay simple and bounded away, and
\[
\disc(p_\delta)=c\,\delta^{2}+O(\delta^{3}),\qquad c\neq0,
\]
with $c>0$ in the real case and $c<0$ in the complex case.  The zero of $\disc(p_\delta)$ at $\delta=0$ is of even order: along the family the discriminant approaches zero from one side and does not change sign; the sign of $c$ distinguishes a real-pair from a complex-pair collision, and the event is detected by the exact zero at $\delta=0$ and the certified $\Theta(\delta^{2})$ decay.
\end{lem}

\begin{proof}
By \cref{prop:nodes} the Hessian at $k^{*}$ is nondegenerate, and $F(k^{*}+x)=H(x,x)+O(\|x\|^{3})$.  Parameterize $\ell_\delta$ so that $s=0$ is its point nearest $k^{*}$, i.e.\ $k(s)=k^{*}+x_\delta+sn$ with $\|x_\delta\|=\delta(1+o(1))$ and $x_\delta/\|x_\delta\|\to\bar x$.  Then for $|s|\le C\delta$, $p_\delta(s)=H(n,n)\,s^{2}+2H(x_\delta,n)\,s+H(x_\delta,x_\delta)+O(\delta^{3}),$ and the quadratic part has two roots with gap $\frac{2\sqrt{H(x_\delta,n)^{2}-H(n,n)H(x_\delta,x_\delta)}}{|H(n,n)|}=\Theta(\delta),$ real or complex conjugate according to the sign of the radicand, whose limit is governed by $D(\bar x)$.  The $O(\delta^{3})$ remainder perturbs these roots by $O(\delta^{2})$, so the corresponding pair of roots of $p_\delta$ has squared gap $\bigl(4D(\bar x)/H(n,n)^{2}\bigr)\delta^{2}\bigl(1+O(\delta)\bigr)$.  The remaining two roots converge to the two simple intersections of $\ell_0$ with the surface away from the node, so they stay simple and bounded away from the pair.  In $\disc(p_\delta)=\lead(p_\delta)^{6}\prod_{i<j}(s_i-s_j)^{2}$ the colliding-pair factor is the squared gap and every other factor is $\Theta(1)$ and smooth in $\delta$, so $\disc(p_\delta)=c\,\delta^{2}+O(\delta^{3})$ with $c\neq0$.  The sign statement is the classical sign rule for real quartics: with the far pair real and simple, $\disc>0$ exactly when the near pair is also real.
\end{proof}

\subsection{Continuation for interactive use}

\begin{lem}[Simple-root continuation]
\label{lem:continuation}
Let $p_t$ be a family of real polynomials of fixed degree four with nonzero leading coefficient, with coefficients depending continuously (resp.\ $C^1$) on $t$, and let $\disc(p_{t_0})\neq0$.  Interface polynomials satisfy the degree hypothesis by \cref{lem:degree}.  Then on a neighborhood of $t_0$ the real roots of $p_t$ vary continuously (resp.\ $C^1$) and preserve their count and ordering.
\end{lem}

\begin{proof}
At a simple root, $p_{t_0}'(s_0)\neq0$ and the implicit function theorem yields a locally unique root branch of the stated regularity.  While $\disc(p_t)\neq0$ persists, branches can neither collide nor pass through complex pairs onto the real line, since either event forces $\disc=0$; hence count and ordering are locally constant.
\end{proof}

\begin{cor}[Certified temporal coherence]
\label{cor:coherence}
Let $t\mapsto(\varepsilon_t,n_t,k_{\parallel,t})$ be a continuous path of interface configurations along which $\disc(p_t)\neq0$.  Then the number, ordering, and flux classification of the modes are constant along the path.  Consequently every change in the transmitted-mode structure during navigation occurs on the discriminant locus, which \cref{alg:solver} monitors.
\end{cor}

\begin{proof}
The count and ordering of the real roots are constant by \cref{lem:continuation}.  Since $\disc(p_t)\neq0$, no root lies at a node of the surface: a line through a node $k^{*}$ has a double root there, because $p_t(s^{*})=F(k^{*})=0$ and $p_t'(s^{*})=\nabla F(k^{*})\cdot n=0$.  Each $k_i(t)$ is therefore a smooth point of the wave surface, where $\ker M(k_i)$ is one-dimensional~\cite{BornWolf1999}, so the polarization $E_i(t)$ and hence $\mathbf S_i(t)\cdot n_t$ vary continuously; by \cref{lem:flux} no flux vanishes at a simple root, so the vector of flux signs is locally constant, and the admitted subset is determined by this sign vector.
\end{proof}

During navigation the solver therefore tracks roots from frame to frame, seeding refinement with the previous roots while keeping certified isolating intervals, and falls back to fresh isolation with escalated precision whenever the event predicate approaches zero.  Interaction acquires an event-driven structure: smooth parameter motion, root continuation, algebraic event, branch reclassification.

\subsection{The main theorem}
\label{sec:main}

\begin{thm}[Main theorem]
\label{thm:main}
Let $\varepsilon=\diag(\varepsilon_1,\varepsilon_2,\varepsilon_3)$ with $0<\varepsilon_1<\varepsilon_2<\varepsilon_3$ be the relative permittivity of a transparent biaxial dielectric at frequency $\omega$, let $F_{\varepsilon,\omega}$ be its Fresnel polynomial~\eqref{eq:fresnel}, and let $m=(u,v)=(\varepsilon_1/\varepsilon_2,\varepsilon_3/\varepsilon_2)$ be the intrinsic coordinate~\eqref{eq:uv}, which lies in the open cell $u<1<v$ of $\Delta$.  Then:
\begin{enumerate}[label=\textup{(\roman*)},nosep]
\item \textup{(Complete invariant.)}  Two strictly biaxial permittivities have wave surfaces related by a rotation and a uniform rescaling of $k$ if and only if they have the same intrinsic coordinate $m$; the factored state $q=(m,\rho,[R],\lambda)$ of~\eqref{eq:state}, with orientation the class $[R]\in\SO(3)/D_2$ over the biaxial cell, therefore carries no representational redundancy.
\item \textup{(Exact singular geometry.)}  The real singular locus of $\{F_{\varepsilon,\omega}=0\}$ consists of exactly the four nodes~\eqref{eq:nodes}, given in closed form; the optic axes and their half-angle $\beta$ from the $k_3$-axis, $\tan^{2}\beta=v(1-u)/(u(v-1))$, hence the line angle $\theta_{\rm OA}=\min(2\beta,\pi-2\beta)$, are exact functions of $m$.
\item \textup{(Certified refraction.)}  For every nonzero interface normal $n$, oriented into the transmitting medium, and tangential wave vector $k_\parallel$, the candidate transmitted wave numbers are the roots of the explicit polynomial $p(s)$ of~\eqref{eq:ps}, which has degree exactly four.  If $\varepsilon$, $n$, $k_\parallel$, and $k_0$ have dyadic-rational entries---with no exact normalization of $n$ or exact orthogonality $k_\parallel\perp n$ required of them---then the number and ordering of the simple real roots of $p$, the presence of any repeated root, the node encounters, and the flux classification of every simple real root are decided by finitely many exact sign evaluations, and along any continuous parameter path on which $\disc(p_t)\neq0$ the real roots vary continuously with constant count and ordering.
\end{enumerate}
\end{thm}

\begin{proof}
(i) The first statement is \cref{prop:chart}.  The ordered eigenvalues are $(u\rho,\rho,v\rho)$, and the eigenframe is determined up to $D_2$ by \cref{sec:chart}.  Hence, at the fixed wavelength under consideration, $q$ determines $\varepsilon$, and $\varepsilon$ determines $q$.

(ii) The first statement is \cref{prop:nodes}.  Substituting $\varepsilon_1=u\varepsilon_2$ and $\varepsilon_3=v\varepsilon_2$ into its angle formula gives $\tan^{2}\beta=v(1-u)/(u(v-1))$.

(iii) The reduction is derived in \cref{sec:reduction}, and $\deg p=4$ by \cref{lem:degree}.  The coefficients of $p$ are integer polynomials in the input, so they are rational on dyadic-rational input.  Sturm's theorem counts the real roots of $p$ in any interval with rational endpoints by finitely many exact sign evaluations, and bisection yields isolating intervals in the order of the roots.  The rational number $\disc(p)$ vanishes if and only if $p$ has a repeated root.  By \cref{prop:certified} these decisions do not depend on the refinement.  A node encounter is a common real root of $p$ and the three cubics $G_j$, decided by the gcd; the flux classification of a simple root is the sign of $p'(s)\,e_2(M(k_\parallel+sn))$ at that root by \cref{lem:flux}, decided by refining its isolating interval.  The continuation statement is \cref{lem:continuation}, applied on a connected path.
\end{proof}

\section{System: Renderer and Linked Navigator}
\label{sec:system}

This section describes the design of the interactive system; the components evaluated in this paper are the exact solver and the learning audits, and rendered output is not evaluated (E5).  The design exposes one state~\eqref{eq:state} through four synchronized views: a material panel (principal permittivities, scale, orientation, wavelength model); the quotient navigator, which draws the intrinsic chart $\Delta$ with its uniaxial edges, isotropic corner, and contour lines of the optic-axis angle $\beta(u,v)$, and, in moduli mode, the ambient Kummer moduli in the coordinates~\eqref{eq:igusa-encoder} with the discriminant and symmetry strata marked; the wave-surface view, showing the quartic, its four nodes drawn directly from~\eqref{eq:nodes}, the optic axes, and the interface line $k(s)$ under inspection; and an appearance view supplied by a renderer.  Because all four views are functions of the same state, a drag in any one of them is a motion in the shape space, and the event structure of \cref{sec:solver} specifies the desired behavior: away from events, root continuation preserves temporal coherence (\cref{cor:coherence}); at an event, the affected branch is reclassified and the crossing can be announced in the navigator rather than surfacing as an unexplained flicker.

\begin{figure}[t]
\centering
\includegraphics[width=\linewidth]{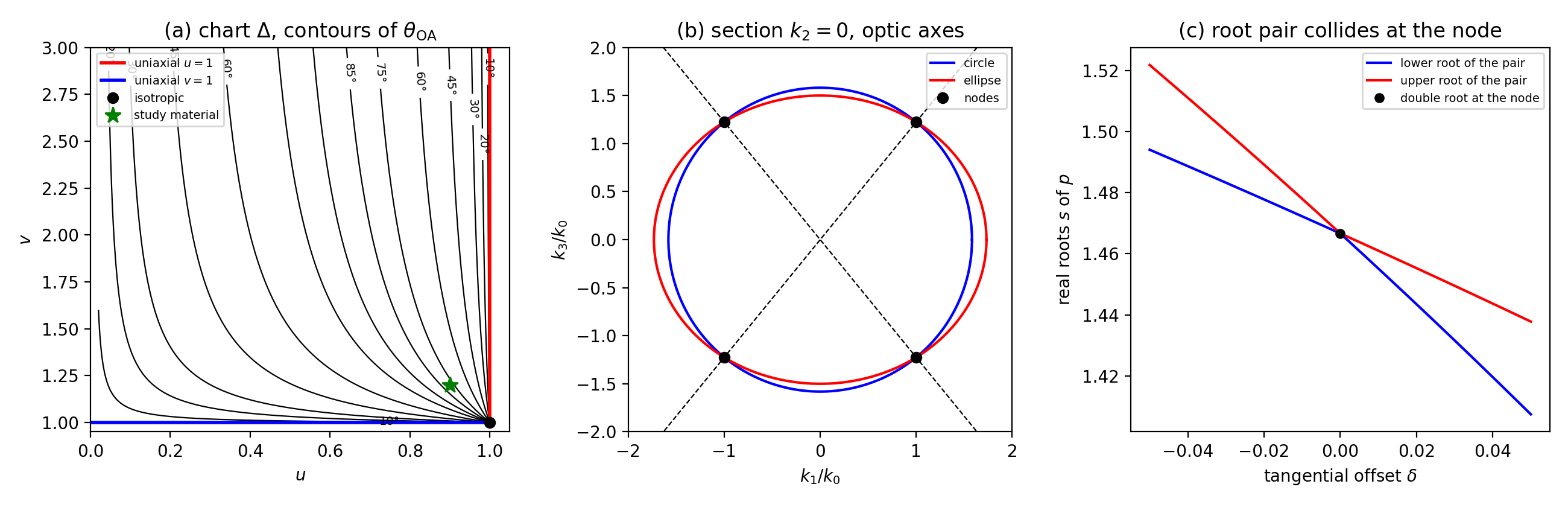}
\caption{The objects the prototype design links, computed from the formulas of \crefrange{sec:fresnel}{sec:solver}; (a) and (b) use the study material $(\varepsilon_1,\varepsilon_2,\varepsilon_3)=(9/4,5/2,3)$, (c) the rational-node material $(16/5,4,36/5)$ of the exact run in E3.  (a) The intrinsic chart $\Delta$ with its uniaxial strata, isotropic corner, and contours of the optic-axis line angle $\theta_{\rm OA}(u,v)$.  (b) The $k_2=0$ section of the wave surface: the circle and ellipse of \cref{prop:nodes} crossing at the four real nodes, with the optic axes.  (c) The transmitted branch of the interface quartic $p(s)$ as the ray is steered toward a node: the root pair collides, which is the event regime of \cref{sec:experiments}.}
\label{fig:pipeline}
\end{figure}

The intended renderer targets homogeneous, nonparticipating transparent crystals.  Interface quantities and polarization transport follow the established birefringence literature~\cite{WeidlichWilkie2008,Steinberg2019}; the biaxial directional solve is \cref{alg:solver}; spectral integration may use a conventional wavelength-sampled path tracer, since analytic spectral acceleration is orthogonal to the present contribution.  Canonical configurations associated with the two dual singular structures of \cref{sec:singular}---conoidal nodes and circular-contact tangent planes---provide natural conical-refraction test cases, with the wave-surface view displaying the marked structure in question.  Consistent with \cref{sec:embedding}, the design treats node events, material degenerations, and interface root events as three visually distinct classes.  Rendering output itself is outside the evaluation reported in E5.

\section{Learning on the Quotient}
\label{sec:learning}

The quotient supplies exact equivalence classes. A class is sampled by choosing a quotient point and applying random elements of $G_{\rm task}$. Membership in a class is then known by construction. We use this ground truth to measure what learned models preserve. The same construction applies to the Kummer family of \cref{sec:kummer}, with sextics and the $\GL_2$ action. We report experiments on the Fresnel family, where every quantity is available in closed form.

\begin{defn}[Invariance defect]
\label{defn:defect}
Let $h:\A\to\R$ be a learned map, let $a\in\A$, and let $g_1,\dots,g_K\in G_{\rm task}$. The \emph{invariance defect} of $h$ at $a$ is the standard deviation of $h(g_1\cdot a),\dots,h(g_K\cdot a)$.
\end{defn}

The defect of a map that factors through the quotient is zero. The defect of a map on raw parameters is an empirical quantity.

\subsection{Invariance audit of a learned regressor}
\label{sec:l1}

The parameter space is the set of strictly biaxial positive definite tensors $\varepsilon$ in world coordinates. The group $G_{\rm task}=\SO(3)\times\R_{>0}$ acts by $(R,t)\cdot\varepsilon=tR\varepsilon R^\top$. The target is the optic-axis line angle $\theta_{\rm OA}$. By \cref{prop:chart} it is a function of $m$ alone.

Quotient points are sampled uniformly in $[0.5,0.98]\times[1.02,2]$. Representatives use Haar-random $R$ and a scale $\rho$ drawn log-uniformly. Training scales lie in $[1,4]$. The raw encoding is the six upper-triangular entries of $\rho R\diag(u,1,v)R^\top$. The augmented raw encoding uses eight random representatives per training class instead of one. The quotient encoding is $(u,v)$. The model is a multilayer perceptron with two hidden layers of width $64$, standardized inputs, and early stopping. The test set consists of $500$ held-out classes. Each class has $16$ representatives with $\rho\in[1,4]$ and $16$ with $\rho\in[4,16]$. All values are means over three seeds. The standard deviation of $\theta_{\rm OA}$ over the sampling distribution is $16.3^\circ$.

\begin{table}[t]
\centering
\caption{Regression of $\theta_{\rm OA}$ in degrees. RMSE is the root-mean-square error. The defect is the mean over test classes of the invariance defect of \cref{defn:defect} with $K=16$. Means and standard deviations over three seeds.}
\label{tab:invariance}
\small
\setlength{\tabcolsep}{4pt}
\begin{tabular}{llcccc}
\toprule
 & & \multicolumn{2}{c}{$\rho\in[1,4]$} & \multicolumn{2}{c}{$\rho\in[4,16]$}\\
Encoding & Classes & RMSE & Defect & RMSE & Defect\\
\midrule
Raw & $200$ & $20.0\pm2.9$ & $10.1\pm3.7$ & $170.8\pm55.3$ & $91.1\pm28.2$\\
Raw, augmented $\times8$ & $200$ & $15.8\pm0.3$ & $2.7\pm0.4$ & $62.2\pm6.0$ & $34.7\pm3.4$\\
Quotient $(u,v)$ & $200$ & $4.7\pm0.9$ & $0$ & $4.7\pm0.9$ & $0$\\
\midrule
Raw & $2000$ & $15.8\pm0.3$ & $2.8\pm0.3$ & $49.1\pm8.6$ & $29.7\pm4.9$\\
Raw, augmented $\times8$ & $2000$ & $11.8\pm0.4$ & $6.2\pm0.1$ & $66.7\pm24.3$ & $45.1\pm14.1$\\
Quotient $(u,v)$ & $2000$ & $1.4\pm0.0$ & $0$ & $1.4\pm0.0$ & $0$\\
\midrule
Raw & $20000$ & $10.5\pm0.7$ & $6.1\pm0.2$ & $41.4\pm5.1$ & $30.5\pm2.8$\\
Quotient $(u,v)$ & $20000$ & $0.5\pm0.0$ & $0$ & $0.5\pm0.0$ & $0$\\
\bottomrule
\end{tabular}
\end{table}

\Cref{tab:invariance} reports the results. On raw tensors the regressor does not become invariant with more data. Its mean in-range defect is between $2.7^\circ$ and $10.1^\circ$ in all settings. Augmentation lowers the in-range error but does not consistently improve the out-of-range error or enforce invariance. Outside the training scales every raw model fails, with errors above $41^\circ$. The quotient regressor has zero defect by construction. Its error does not depend on the scale range. With $20000$ classes its error is $0.5^\circ$, a factor $21$ below the best raw model tested.  A constant predictor has error $16.3^\circ$, and the closed-form angle is the exact reference; the experiment audits invariance rather than the need to learn a known formula.

\subsection{Learned root-count predicates}
\label{sec:l2}

The second task is the root count of \cref{thm:main}(iii). The input is an interface configuration. The label is the number of distinct real roots of $p$, which is $0$, $2$, or $4$. We fix $n=e_3$ and $k_0=1$. Materials are sampled as in \cref{sec:l1} with $\rho\in[1,4]$. Half of the tangential vectors are uniform in the disk of radius $2$. The other half lie at distance $\delta$ from the tangential trace of a random real node, with $\delta$ log-uniform in $[10^{-7},10^{-1}]$. Inputs are read as dyadic rationals. Labels are computed exactly.  Write $p(s)=as^{4}+bs^{3}+cs^{2}+ds+e$ with $a>0$ and put $P=8ac-3b^{2},\qquad D=64a^{3}e-16a^{2}c^{2}+16ab^{2}c-16a^{2}bd-3b^{4}.$ If $\disc(p)<0$, then $p$ has exactly two distinct real roots.  If $\disc(p)>0$, then $p$ has four distinct real roots when $P<0$ and $D<0$, and no real roots otherwise~\cite{Rees1922}.  Inputs with $\disc(p)=0$ are discarded. On $300$ random inputs these labels agree with $60$-digit root finding. There are $40000$ training and $20000$ test configurations.  Hyperparameters, software versions, hardware, and code are given in Section~SM5 of the supplementary materials.

For fixed $n$ the group $G_{\rm task}$ consists of rotations about $n$ and the scalings $(\varepsilon,k_\parallel)\mapsto(t\varepsilon,\sqrt t\,k_\parallel)$. By the second identity in the proof of \cref{prop:chart}, both preserve the label. The raw encoding is the six entries of $\varepsilon$ and the two entries of $k_\parallel$. The invariant encoding rotates $k_\parallel$ onto the positive $k_1$-axis and divides by $\tau=\tr(\varepsilon)/3$; it is defined for $k_\parallel\neq0$, which holds with probability one under the sampling used. It consists of the six entries of the rotated $\varepsilon/\tau$ and of $\|k_\parallel\|/\sqrt\tau$. The models are a multilayer perceptron with two hidden layers of width $128$ and a histogram gradient-boosted tree ensemble.

\begin{table}[t]
\centering
\caption{Error rates in percent of learned root-count predicates. The last column contains the $523$ near-node test configurations whose exact count is not $4$. Means and standard deviations over three seeds.}
\label{tab:predicate}
\small
\setlength{\tabcolsep}{4pt}
\begin{tabular}{llcccc}
\toprule
Model & Encoding & All & Uniform & Near node & Near node, count $\neq4$\\
\midrule
Perceptron & Raw & $5.9\pm0.4$ & $6.1\pm0.6$ & $5.8\pm0.4$ & $87.8\pm2.5$\\
Perceptron & Invariant & $2.6\pm0.1$ & $1.2\pm0.1$ & $3.9\pm0.1$ & $73.3\pm1.4$\\
Boosted trees & Raw & $8.3\pm0.0$ & $10.6\pm0.0$ & $6.1\pm0.0$ & $89.8\pm0.6$\\
Boosted trees & Invariant & $3.2\pm0.1$ & $2.0\pm0.1$ & $4.3\pm0.0$ & $77.9\pm0.5$\\
\bottomrule
\end{tabular}
\end{table}

\Cref{tab:predicate} reports the results. The invariant encoding lowers the error on uniform configurations by a factor of five for both models. Near a node $95\%$ of the test configurations have count $4$. The remaining configurations are misclassified at rates between $73\%$ and $90\%$. On them the count is determined at the scale $\delta$ of the offset, and the tested models usually predict the majority count of the neighborhood. The exact predicate of \cref{alg:solver} defines the labels and has error zero on the represented inputs.

The two experiments separate two effects. The quotient removes nuisance variation and improves both accuracy and invariance. The event structure at the scale of the offset is not learned by either encoding. It is resolved by the exact predicate.

\section{Experiments}
\label{sec:experiments}

We evaluate the two halves of the construction separately: the quotient as a representation (E1), and the certified solver as a rendering component (E2--E5).

\subsection{E1: quotient geometry and solver difficulty}
We test the couplings between the material quotient and the interface solver, and we report the negatives together with the positives.

\emph{(A) Stratum distance does not predict generic-ray difficulty.}  Fixing $u=9/10$ and sweeping $v=1+d$ toward the uniaxial stratum, we draw $250$ random tangential wave vectors at a fixed interface normal for each material and measure the normalized discriminant $|\disc(p)|/\|c\|_\infty^{6}$ exactly, where $c$ is the coefficient vector of $p$; this quantity is invariant to rescaling the polynomial, not the variable, and it is a proximity measure to the event locus rather than a condition number.  The median decays like $d^{1.25}$ for $d\ge0.02$ and is flat ($d^{0.07}$) below (\cref{fig:coupling}, left); a single fit over all $d$ returns the misleading exponent $0.49$.  The fraction of rays with normalized discriminant below $10^{-12}$ is zero at all sampled $d$.  The structural reason is visible in \cref{lem:factorizations}: even exactly on the stratum the quartic factors into two quadrics whose four roots are distinct for a generic ray, so difficulty concentrates on a ray-localized set rather than spreading over the material.  The naive hypothesis ``distance to the stratum predicts difficulty'' is therefore refuted in bulk form.

\emph{(B) What does hold: localization, gauge invariance, and structured interpolation.}  First, the location and extent of the dangerous ray set are functions of the quotient point in closed form: the nodes~\eqref{eq:nodes} and the half-angle $\tan^2\beta=v(1-u)/(u(v-1))$ place the difficult directions exactly, and $\disc(p)$ tracks the approach to a node as $\delta^{2}$, the order proved in \cref{lem:order} and confirmed in E3; the quotient thus \emph{predicts where} difficulty lives, enabling precomputation of event geometry per material.

Second, redundancy removal in the strict sense---same physical state, different descriptions---is tested through the eigenframe gauge of \cref{sec:chart}: the eigendecompositions $(\Lambda,R)$ and $(\Lambda,RP)$ with $P\in D_2$ describe the same world tensor, since $R\Lambda R^\top=(RP)\Lambda(RP)^\top$ for the sign flips.  Linearly interpolating between two states in these redundant coordinates produces intermediate tensors that depend on which representative is chosen at each endpoint, so an arbitrary description choice leaks into the rendered result; the quotient encoding maps all representatives of a state to the same point by construction, and no downstream computation can see the choice.

Third, interpolation across \emph{orientations} probes something different, and we state it as such: the endpoints are genuinely distinct physical states sharing the intrinsic class $(u,v,\rho)$, and the comparison is between structure-preserving interpolation on the stratified state space and linear blending in world-tensor coordinates.  For orientations $70.5^\circ$ apart about the axis $[1,0,1]/\sqrt2$, the two normalized-discriminant curves are quantitatively distinct: across the sampled path, the raw-to-structured median ratio ranges from $0.298$ to $1.381$.  Both nevertheless remain more than eight orders of magnitude above the $10^{-12}$ event threshold (\cref{fig:coupling}, center).  The matrix path already fabricates intermediate materials: its eigenvalues drift from $(2.25,2.5,3)$ to $(2.30,2.54,2.91)$ and $\theta_{\rm OA}$ from $78.5^\circ$ to $85.0^\circ$ at $t=1/2$.  Blending across a $90^\circ$ rotation about $k_3$ exhibits the failure mode fully: along the matrix path the eigenvalues mix and the optic-axis line angle collapses from $\theta_{\rm OA}=78.5^\circ$ to $0^\circ$ as the path passes through the exactly uniaxial tensor $\diag(2.375,2.375,3)$ at $t=1/2$, then recovers, whereas the structured path holds $(u,v,\rho)$ fixed and transports the orientation class, so $\theta_{\rm OA}$ is constant by construction (\cref{fig:coupling}, right).  Thus, whether matrix blending is mildly quantitative or class-destroying depends on the rotation path, while the structured path preserves the material class uniformly.  The fabricated intermediate materials are physically different crystals---wrong birefringence, coincident optic axes---so the artifact is visual, not merely numerical.

\begin{figure}[t]
\centering
\includegraphics[width=\linewidth]{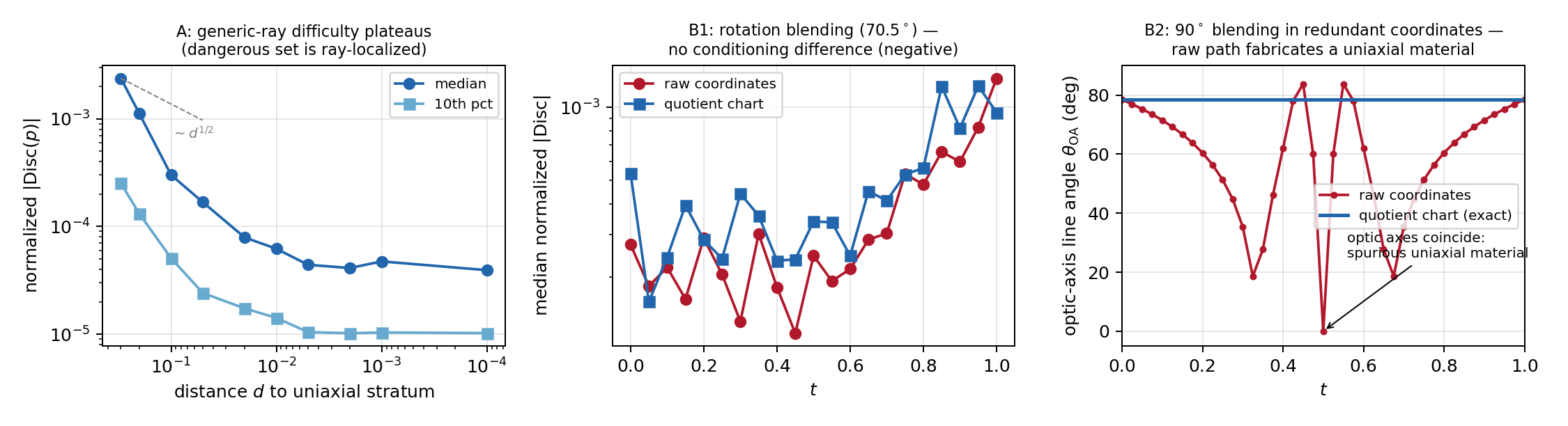}
\caption{Coupling between quotient geometry and the solver.  Left: generic-ray difficulty versus distance to the uniaxial stratum decays as $d^{1.25}$ for $d\ge0.02$ and is flat below---bulk stratum distance does not predict difficulty.  Center: orientation blending about $[1,0,1]/\sqrt2$ produces distinct normalized-discriminant curves, but both remain far from the event threshold.  Right: interpolating between two orientations of the same material class, linear blending of world tensors fabricates a spurious uniaxial material (optic axes coincide, $\theta_{\rm OA}\to0$) while the structured path on the stratified state space holds the material class exactly.}
\label{fig:coupling}
\end{figure}

\subsection{E2: solver correctness}
The correctness checks are reported in Section~SM7 of the supplementary materials.

\subsection{E3: robustness near degeneracies}
We report a controlled study at the most delicate regime: an interface line driven toward an optic-axis node.  The crystal has principal permittivities $(\varepsilon_1,\varepsilon_2,\varepsilon_3)=(9/4,\,5/2,\,3)$, rotated by a rational orthogonal matrix so that the interface quartic is generic; the node is the closed-form point of~\eqref{eq:nodes}.  For each offset $\delta\in[10^{-15},10^{-1}]$ we draw $200$ random tangential directions from one seeded generator and place $k_\parallel$ at distance $\delta$ from the node's tangential trace; every input is stored as the exact dyadic rational represented by its IEEE value, so that the exact number of real roots (a Sturm count) and the exact sign of $\disc(p)$ are available as ground truth for the represented state.  The rotated tensor and its matrix $\Gamma_\varepsilon$ are formed in exact rational arithmetic; only $k_\parallel$ is rounded.  The node itself is irrational and its tangential trace is stored as the binary64 point nearest to it, within $2\cdot10^{-16}$; this is the represented state a renderer has, and it is the only representation error in the run.  A separate exact run uses the material $(\varepsilon_1,\varepsilon_2,\varepsilon_3)=(16/5,4,36/5)$, whose node $(6/5,0,8/5)$ is rational, so that the node's trace is exact and configurations can be placed exactly on the event locus.  These inputs are rational but not dyadic, which the solver accepts.  This run uses the same $200$ directions for every $\delta$.  A dyadic fixture with an exact node is $\varepsilon=\diag(5,25/4,45/4)$, $k^{*}=(3/2,0,2)$, $n=(1,0,1)$, $k_\parallel=(-1/4,0,1/4)$, for which $p(s)=\tfrac{5}{512}(4s-7)^{2}(4s+7)(52s+101)$.  Three solvers are compared: a floating-point quartic solve (companion-matrix roots with an imaginary-part tolerance), an iterative Newton solve seeded on a grid with root deduplication (a stand-in for unstructured nonlinear solving), and the certified solver of \cref{alg:solver}.

Along every sampled direction the ground truth near this node is four real roots, which Poynting-flux classification resolves into exactly two forward (transmitted) and two backward modes in all sampled configurations; the complex-pair branch of \cref{lem:order} ($D(\bar x)<0$) does not occur for the sampled geometry, in which $n$ is the interface normal and the offset is tangential, and the colliding pair is the forward pair in every configuration.  Ordering the singular values of $M(k_i)$ as $\sigma_1\ge\sigma_2\ge\sigma_3$, the polarization diagnostic $\sigma_2/\sigma_1$ at the near pair is proportional to $\delta$, with constant between $0.15$ and $0.23$ over the sampled directions.  All observed baseline failures are \emph{undercounts}, by two different mechanisms.  The companion-matrix solve returns the colliding real pair as a complex pair and drops it: in every observed failure it returns two roots, both backward, so both forward transmitted modes are lost.  The Newton solve stays real but merges the pair, since its deduplication tolerance $10^{-7}$ exceeds the root gap for $\delta\le10^{-7}$: in every observed failure it returns three roots with one forward mode, so one of the two forward modes is lost.  At $\delta=10^{-9}$ these counts are $106$ of $200$ and $196$ of $200$ respectively.  \Cref{tab:robustness} and \cref{fig:robust} summarize the outcome.  The baselines return the correct count in all sampled cases for $\delta\ge10^{-6}$, and silent root-count errors switch on at $\delta=10^{-7}$, where the exact root gap is approximately $0.22\,\delta$ over the central asymptotic range and reaches the accuracy floor of a nearly double root in the companion-matrix solve ($\approx10^{-8}$, consistent with square-root error amplification, not with float64 resolution itself).  The certified solver returns the exact count throughout.  For $10^{-13}\le\delta\le10^{-2}$, $\min_\varphi|\disc(p)|/\delta^{2}$ remains within $0.13\%$ of $1003.5$; over the full sampled sweep the discriminant decreases by more than twenty-six orders of magnitude, with $\disc(p)>0$ in every configuration as \cref{lem:order} requires.  At $\delta=10^{-14}$ and $10^{-15}$ the ratio falls by $1.2\%$ and $12\%$; this is the binary64 representation error of the node, which is a fifth of the offset at $\delta=10^{-15}$.  The exact rational-node sweep removes that floor: $\min_\varphi|\disc(p)|/\delta^2$ is constant to four significant digits for $10^{-15}\le\delta\le10^{-3}$.  Placing $k_\parallel$ on the exact node's trace exercises the event branch of \cref{alg:solver}: the solver returns $\disc(p)=0$, the double root exactly with multiplicity two, and two simple roots.  At $k^*$ the largest singular value is $6.4$ and the other two are below $10^{-15}$, at the level of binary64 round-off, so the two-dimensional kernel identifies the configuration as a conical node (\cref{fig:pipeline}(c)).

\begin{table}[t]
\centering
\caption{Silent root-count error rates near an optic axis for the study material ($200$ seeded directions per $\delta$; exact ground truth on the represented dyadic-rational inputs; solver constants in Section~SM6 of the supplementary materials).  Flux classification shows every missed root is a forward transmitted mode.}
\label{tab:robustness}
\begin{tabular}{lccc}
\toprule
 & $\delta=10^{-7}$ & $\delta=10^{-9}$ & $\delta=10^{-12}$\\
\midrule
Newton iteration (seeded, deduplicated) & $100.0\%$ & $96.0\%$ & $95.5\%$\\
Float64 quartic roots, tol $10^{-10}$ & $33.0\%$ & $48.5\%$ & $53.0\%$\\
Float64 quartic roots, tol $10^{-8}$ & $13.5\%$ & $25.5\%$ & $27.0\%$\\
Certified (exact Sturm + discriminant) & $0\%$ & $0\%$ & $0\%$\\
\bottomrule
\end{tabular}
\end{table}

\begin{figure}[t]
\centering
\includegraphics[width=\linewidth]{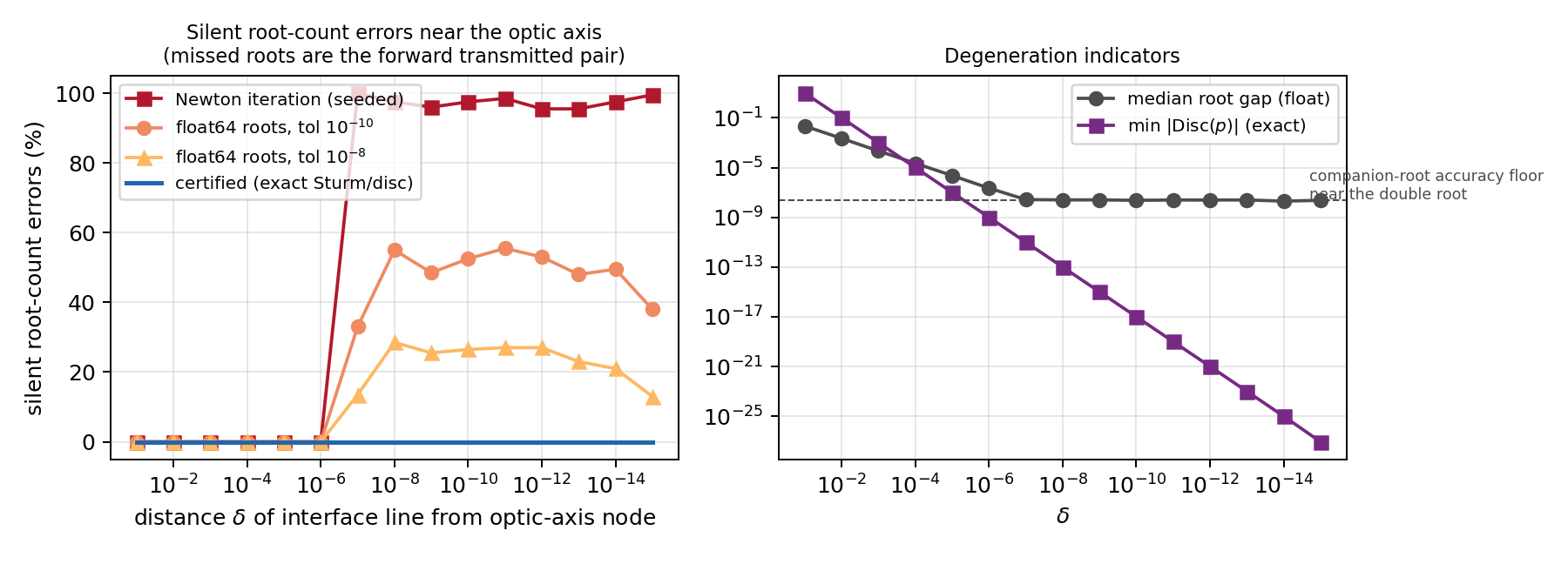}
\caption{Left: silent root-count error rates as the interface line approaches the node (certified solver: zero throughout); the missed roots are the forward transmitted pair.  Right: for the study material, the exact discriminant follows the predicted $\delta^{2}$ law; the plotted root gap is the floating-point one, and the exact gap, computed separately, is approximately $0.22\,\delta$ over the central asymptotic range; at the smallest offsets, representation error in the stored node becomes visible.  The floating-point root gap saturates at the accuracy floor of the nearly double root.}
\label{fig:robust}
\end{figure}

\subsection{E4: runtime}
The runtime measurements are reported in Section~SM8 of the supplementary materials.

\subsection{E5: scope}
The experiments above evaluate the quotient and the interface solve.  Rendered images are not evaluated in this paper.  A natural test material is aragonite at $589\,$nm, with principal indices $(1.530,\,1.681,\,1.686)$, hence $\varepsilon\approx(2.341,\,2.826,\,2.843)$, $(u,v)\approx(0.828,1.006)$, and optic-axis line angle $\theta_{\rm OA}\approx19^\circ$; it is the crystal of Lloyd's laboratory demonstration of conical refraction.

\section{Conclusion}
\label{sec:conclusion}

When an equivalence, a singularity, or an event structure is known exactly, it should be encoded before a numerical or learned method is asked to rediscover it.  We have developed this principle into a representation---canonical algebraic shape spaces with task-aware quotients, invariant coordinates, and explicit strata---and tested its mathematical and computational core on a visual-computing case study in which the geometry and physics are both nontrivial.  Kummer quartics supply the exact quotient; Fresnel wave surfaces supply the interface problem, and \cref{thm:kummer} identifies each wave surface with the Kummer surface of an explicit split Jacobian; \cref{thm:main} delivers the chart as a complete invariant, the singular geometry in closed form, and refraction as a certified univariate quartic whose discrete decisions are consequences of the configuration rather than hopes about convergence.

What the combination establishes, as measured in E1 and E3, is a refined computational pattern.  Approaching a uniaxial stratum does not uniformly drive the generic-ray discriminant toward zero---distance to a material stratum does not predict proximity to the multiple-root locus---but it predicts \emph{where} the delicate configurations live, since the nodes and the line angle $\theta_{\rm OA}$ are closed-form functions of the quotient point, and at what order events are approached (\cref{lem:order}); it removes representation-induced artifacts exactly, since equivalent descriptions can no longer produce different results, in interpolation or otherwise; and certified predicates turn the predicted event loci from silent failure modes into detected events.  The interface problem makes the pattern measurable; the framework makes it portable to any family algebraic enough to expose invariants, representative reconstruction, and a discriminant---rational curves and surfaces, algebraic CAD primitives, dispersion surfaces beyond the dielectric case, caustic and focal families, and hybrid algebraic--neural representations.

The learning experiments of \cref{sec:learning} measure the same pattern for learned models.  Quotient coordinates make a regressor invariant by construction and reduce its error by a factor of $21$ relative to raw tensors.  Learned root-count predicates misclassify most near-node configurations whose count differs from the local majority, and the exact predicate resolves them.  The same audits apply to the Kummer family, with sextics, the $\GL_2$ action, and the classifiers of~\cite{ShaskaShaska2025}.

The limitations are equally definite.  The construction applies to families with usable invariants and reconstruction, not to shape categories at large; the full Kummer quotient is singular and is handled as a stratified object, not a Euclidean latent; and the physical model is the clean homogeneous transparent dielectric---dispersion enters only through wavelength sampling, and absorption, optical activity, and spatial inhomogeneity are outside the present model.  Certification has a price: exact isolation and adaptive precision can lose to a well-behaved floating-point solve in generic configurations, and its advantage concentrates near branch events and in applications where silent mode misclassification is unacceptable.  E4 quantifies this balance for the prototype; renderer-integrated timings remain to be measured.

\bibliographystyle{amsplain}
\bibliography{sh-180}

\end{document}